\documentclass[journal]{IEEEtran}
\usepackage{amsmath,amssymb,amsfonts,amsthm}
\usepackage{amsthm}
\usepackage{mathrsfs}
\usepackage{algorithmic}
\usepackage{graphicx}
\usepackage{epstopdf}
\usepackage{color}
\usepackage{textcomp}
\usepackage{dsfont}
\usepackage{verbatim}
\usepackage[linesnumbered,ruled]{algorithm2e}
\usepackage[colorlinks=false,hidelinks]{hyperref} 
\usepackage{tabularx}
\usepackage{url}
\usepackage{upgreek}

\newtheorem{theorem}{Theorem}[section]

\newtheorem{proposition}{Proposition}[theorem]

\newtheorem{lemma}{Lemma}[theorem]
\newtheorem{definition}{Definition}[theorem]
\newtheorem{remark}{Remark}
\newcommand{\Xext}{\mathbb{X}_{\mathrm{ext}}}

\usepackage{style/hybridsystem}
\usepackage{style/nomenclature_macros}

\begin{document}
%
\title{Control Lyapunov Functions for Compliant Hybrid Zero Dynamic Walking}

%
%

\author{Jenna Reher and 
        Aaron D. Ames
\thanks{Jenna Reher is with the Department of Mechanical and Civil Engineering, California Institute of Technology, Pasadena, CA 91125, {\tt\small jreher@caltech.edu}.}
\thanks{Aaron D. Ames is with the Department of Mechanical and Civil Engineering, California Institute of Technology, Pasadena, CA 91125, {\tt\small ames@caltech.edu}.}
}

%
%

\markboth{IEEE TRANSACTIONS ON ROBOTICS, VOL.~XX, NO.~XX, MONTH~2021
}%
{Reher \MakeLowercase{\textit{et al.}}: Control Lyapunov Functions for Compliant Hybrid Zero Dynamic Walking}
%



\maketitle

\begin{abstract}
    The ability to realize nonlinear controllers with formal guarantees on dynamic robotic systems has the potential to enable more complex robotic behaviors---yet, realizing these controllers is often practically challenging.
    %
    To address this challenge, this paper presents the end-to-end realization of dynamic bipedal locomotion on an underactuated bipedal robot via hybrid zero dynamics and control Lyapunov functions. 
    A compliant model of Cassie is represented as a hybrid system to set the stage for a trajectory optimization framework. 
    With the goal of 
    achieving 
    a variety of walking speeds in all directions, a library of compliant walking motions is compiled and then parameterized for efficient use within real-time controllers. 
    Control Lyapunov functions, which have strong theoretic guarantees, are synthesized to leverage the gait library and coupled with inverse dynamics to obtain optimization-based controllers framed as quadratic programs. 
    %
    It is proven that this controller provably achieves stable locomotion; this is coupled with a theoretic analysis demonstrating useful properties of the controller for tuning and implementation. 
    The proposed theoretic framework is practically demonstrated on the Cassie robot, wherein 3D walking is achieved through the use of optimization-based torque control.
    %
    The experiments highlight robotic walking at different speeds and terrains, illustrating the end-to-end realization of theoretically justified nonlinear controllers on dynamic underactuated robotic systems. 
\end{abstract}


\IEEEpeerreviewmaketitle

\section{Introduction} \label{sec:introduction}

Robotic bipedal locomotion has 
has made impressive strides in recent years 
as humans increasingly look to augment their natural environments with intelligent machines.
In order for bipedal robots to navigate the often unstructured environments of the world and perform tasks, they must first have the capability to dynamically, reliably, and efficiently locomote. 
However, due to the inherently hybrid and underactuated nature of dynamic bipedal walking, achieving experimental success is a delicate balance between developing accurate locomotion models, trajectory planners, and feedback control while maintaining computationally tractable implementations. 
The objective of this work is then to 
develop experimentally realizable optimization-based controllers
that can leverage the full-body dynamics of a robotic platform, including compliance and underactuation, to stabilize dynamic locomotion.

\begin{figure}[!t]
    \centering
    \includegraphics[width=1\columnwidth]{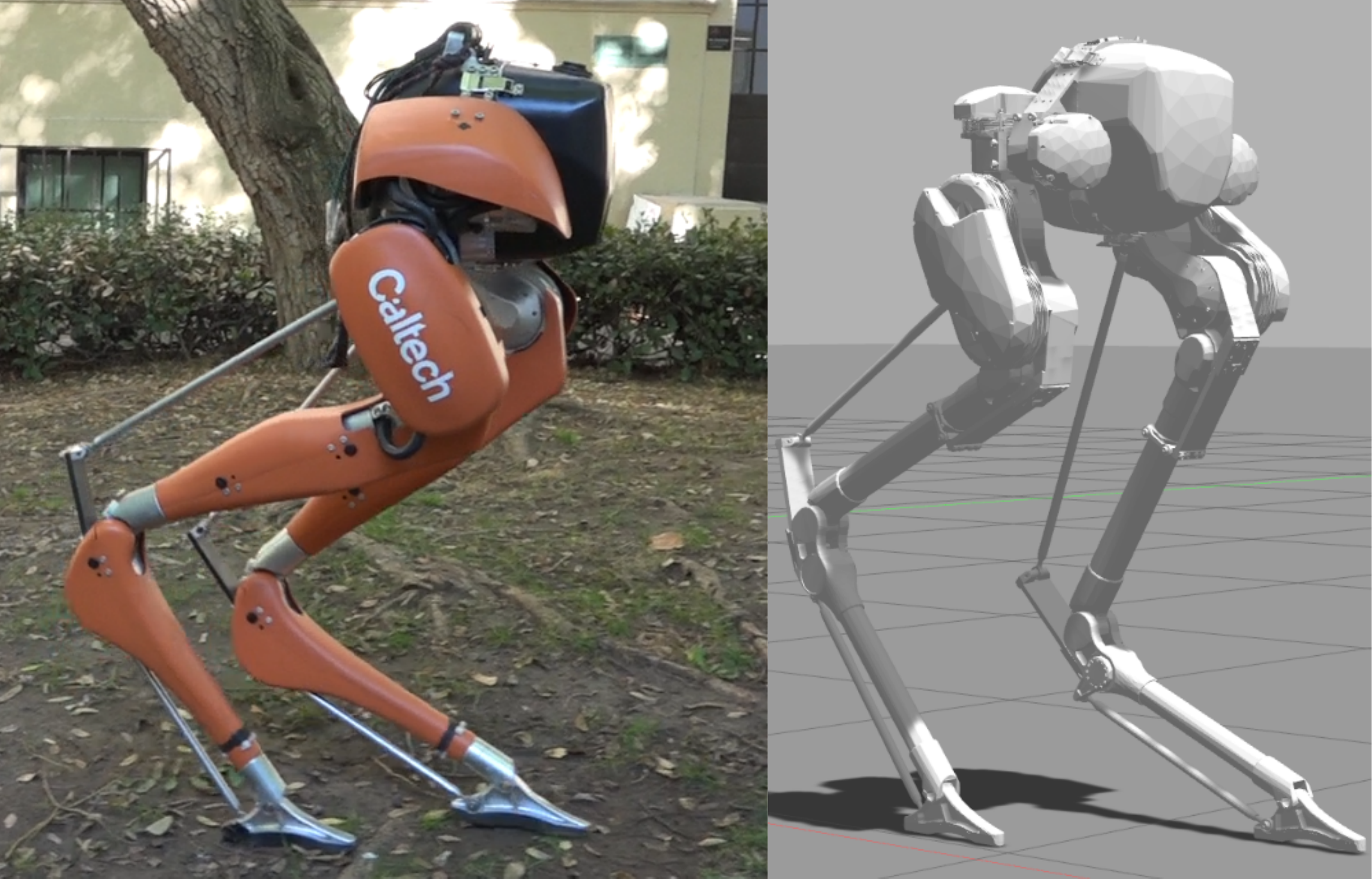}
    \caption{The Caltech Cassie biped walking outdoors and in a Gazebo simulation while using a version of the optimization-based \eqref{eq:id-clf-qp-plus} controller.}
    \vspace{-5mm}
    \label{fig:cover}
\end{figure}

To address the difficult nature of full-body locomotion planning, 
a significant subset of the bipedal robotics literature mitigates the complexity of humanoids and bipeds by viewing walking as a problem wherein the real world dynamics are assumed to be governed by the evolution of a simpler system, such as a LIP models (Linear Inverted Pendulum \cite{kajita2003biped, gong2020angular}), SLIP models (Spring Loaded Inverted Pendulum \cite{rezazadeh2015spring}), and the ZMP (Zero Moment Point \cite{vukobratovic2004zero}). 
These methods can reduce computational complexity for fast planning and experimental success. 
Despite its viability in practical implementation, this local representation of the system can limit the agility of behaviors and compromise energy efficiency, and may require additional optimization to ensure viable walking \cite{Dalibardwhole}.  
In order to realize these reduced-order behaviors on actual robots, the motion must be transcribed into the full-order dynamics typically through inverse kinematics, or inverse dynamics  
to compute control inputs at each instant \cite{kajita2003biped}. 

Model-based torque control methods 
can help enable dynamic and compliant motion of robots while achieving remarkable control performance. However, implementing such techniques on floating base robots is non-trivial due to model inaccuracy, phases of underactuation, dynamically changing contact constraints, and possibly conflicting objectives for the robot \cite{ames2013towards}. Unlike their classical counterparts, optimization-based approaches of handling these control problems allow for the inclusion of physical constraints that the system is subject to \cite{posa2014direct, betts2002practical}. Partially as a consequence of this feature, quadratic programming (QP) based controllers have been increasingly used to stabilize real-world systems on complex robotic platforms without the need to algebraically produce a control law or enforce convergence guarantees \cite{koolen2016design, herzog2016momentum, feng2015optimization}.

Inverse dynamics is a widely used method to approach model-based controller design for achieving a variety of motions and force interactions, typically in the form of task-space objectives.  Given a target behavior, the dynamics of the robotic system are inverted to obtain the desired torques. In most formulations, the system dynamics are mapped onto a support-consistent manifold using methods such as the dynamically consistent support null-space \cite{sentis2007synthesis}, linear projection \cite{aghili2005unified}, and orthogonal projection \cite{mistry2010inverse}. When prescribing behaviors in terms of purely task space objectives, this is commonly referred to as task- or operational-space control (OSC) \cite{khatib1987unified}.  In recent work, variations of these approaches have been shown to allow for high-level tasks to be encoded with intuitive constraints and costs in optimization-based controllers \cite{apgar2018fast, kuindersma2016optimization, feng2015optimization, koolen2016design, herzog2016momentum}.
If a plan wasn't designed for the full-order system, solving such inverse problems does not imply feasibility of future inverse problems in the trajectory \cite{zucker2015general}. 

One area of model-based planning and control which is particularly difficult to directly address is passive compliance in locomotion. 
Some of the earliest inclusions of compliant hardware on bipedal robots was with spring flamingo and spring turkey \cite{hunter1991comparative}, with more recent examples being MABEL \cite{park2011identification}, DURUS \cite{reher2016durusmulticontact}, and ATRIAS \cite{rezazadeh2015spring}. 
One of the latest robots available to researchers exhibiting compliant leg structures is the Cassie biped (shown in \figref{fig:cover}), which is the experimental platform considered in this work. 
From a mathematical standpoint, compliance can increase numerical stiffness and model uncertainty, and can make finding walking behaviors that satisfy stability constraints more difficult. 

The Hybrid Zero Dynamics (HZD) framework \cite{westervelt2018feedback} has demonstrated success in developing controllers for highly underactuated walking behaviors while considering underactuation on the full-order robotic system. 
The basis of the HZD approach is the restriction of the full-order dynamics of the robot to a lower-dimensional attractive and invariant subset of its state space, the \emph{zero dynamics surface}, via outputs that characterize this surface.
If these outputs are driven to zero, then the closed-loop dynamics of the robot are described by a lower-dimensional dynamical system that can be ``shaped'' to obtain stability. 
In the context of robotic implementations, HZD has enabled a wide variety of dynamic behaviors such as multicontact humanoid walking \cite{reher2016durusmulticontact}, compliant running \cite{sreenath2013embedding}, 3D bipedal walking with point-feet \cite{ramezani2014performance}, and locomotion on a variety of planar walking robots \cite{chevallereau2003rabbit,grizzle2009mabel,ramezani2014performance}.
%
To date, the design of controllers to render a stable zero dynamics manifold have most often been tied to feedback linearization of the transverse dynamics. 
It was shown in \cite{ames2014rapidly} that through the use of a class of control Lyapunov functions (CLF)s a wide class of controllers can be designed to create rapidly exponentially convergent hybrid periodic orbits for bipeds \cite{westervelt2018feedback, grizzle2014models}. 
It was also shown that CLFs can be posed as a QP, where convergence is enforced as an inequality \cite{ames2013towards,ames2014rapidly}. 

To enjoy the theoretic guarantees enjoyed by CLFs used in the context of HZD, sufficiently fast convergence is needed.  Model-based torque controllers often can't produce sufficient convergence, especially in the unmodeled dynamics---which is especially acute for compliant systems.  
One approach to address this conflict is to relax convergence guarantees, which allows (local) drift in the control objectives to accommodate feasibility. Utilizing this heuristic, CLFs have since been used to achieve dynamic locomotion on robotic systems both in simulation \cite{nguyen2015optimal, hereid2014embedding, xiong2018coupling} and only in very limited cases on hardware for planar robots \cite{galloway2015torque}. 
While high level task-space controllers based on inverse dynamics approaches pose similar problems as CLF-QPs, they have traditionally not been formulated in the same way. In implementations of CLF-QPs the vector fields associated with robotic systems are typically utilized, which can involve costly computations. Alternatively, in task based controllers, the dynamics are an equality constraint which are affine with respect to the system accelerations, inputs, and constraint forces. For these task-based QP controllers, objectives are driven towards their targets through design of the cost function such as PD feedback on the output dynamics \cite{feng2015optimization} or an LQR control cost \cite{kuindersma2014efficiently, kuindersma2016optimization}. 
\vspace{-5mm}

\begin{figure*}[t!]
	\centering
	\includegraphics[width= 1\textwidth]{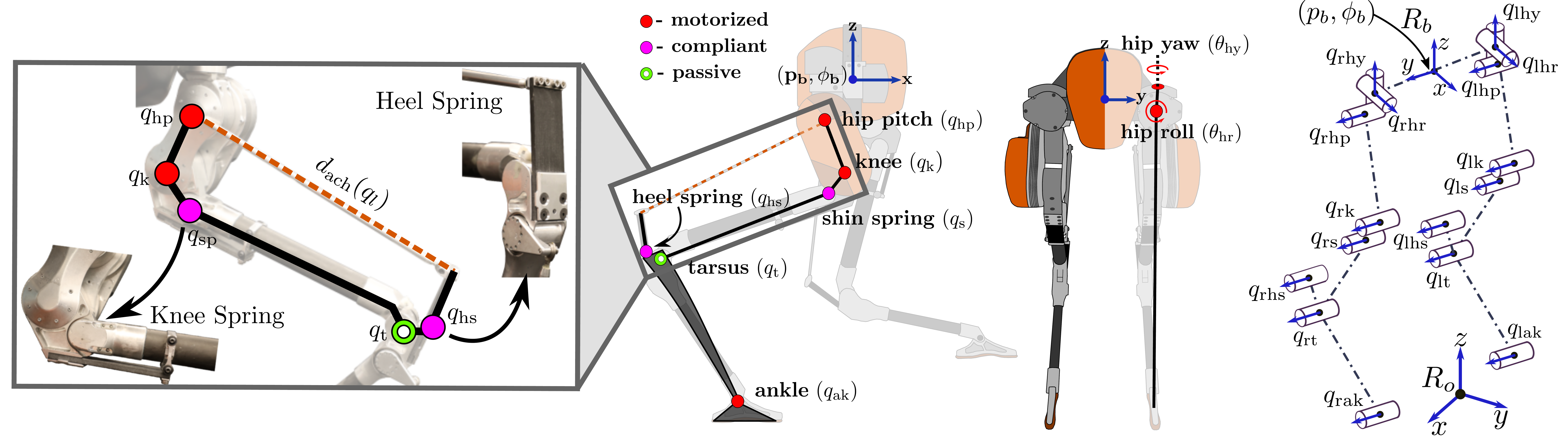}
	\caption{The configuration coordinates of the Cassie robot, showing the compliant leg mechanism and associated springs, motors, and passive joints. 
    }
	\label{fig:cassie_configuration}
\end{figure*}

\subsection*{Contributions}
%
This paper develops a novel instantiation of model-based CLF-QP controller, theoretically establishes stability property and practically demonstrates the methodology experimentally to achieve dynamic 3D underactuated walking. 
This will be realized for HZD walking over a range of walking speeds through the use of a motion library built from a series of HZD trajectory optimizations that can accurately capture passive compliance. 
To achieve these results, this work takes inspiration from, and unifies with, exiting formulations in the area of bipedal locomotion that have proven successful in practice. 
The key results developed in this work are summarized below:
\begin{itemize}
    \item A compliant locomotion model for 
    underactuated 3D bipedal robots 
    is developed and then shown to accurately capture the passive compliant dynamics experimentally on Cassie. 
    The resulting hybrid system model is used to develop a motion library of HZD walking behaviors that leverage a robots full-body dynamics including its compliance. 
    While motion libraries for sagittal motions under the assumption of sufficient rigidity have been realized on Cassie \cite{gong2019feedback} and other robots \cite{da2017supervised}, this work and the preliminary experiments in \cite{reher2020inverse} are the first to consider a \emph{compliant motion library} for the robot.
    %
    %
    \item 
    A novel nonlinear optimization-based controller is presented, the ID-CLF-QP, posing CLFs in a inverse dynamics formulation.  We establish that this class of controller provably yields stable walking.   We demonstrate the benefits of this approach, which introduces additional decision variables to the traditional \eqref{eq:clf-qp} in order to pose an inverse dynamics problem with a CLF convergence constraint. 
    %
    \item 
    The formal controller formulated in this paper can relaxed to enable its robust implementation on hardware. This is implemented on Cassie via torque control, and utilizing the compliant motion library.  
    The resulting experiments demonstrate the \emph{first successful experimental realization of a CLF controller on a 3D biped} in the literature. 
\end{itemize}

The contributions of this paper result in a 
model-based controller 
that is theoretically and practically able to \textit{leverage the compliance} of Cassie \textit{for all motion primitives} in the library. This is demonstrated experimentally on Cassie, with all source code made available for the optimization, controller, and simulation on Github \cite{papergithub}. Locomotion on Cassie is demonstrated over a wide-range of speeds on different terrains, from flat ground to up and down slopes, to over grass and roots.  

\section{Robotic Model} \label{sec:robotmodel}
We begin by introducing hybrid system models of bipedal robots. This includes developing the unpinned continuous dynamics and the discrete dynamics that occur at foot strike. While these concepts are general, we will illustrate them constructively with Cassie to root them practical application. 

The Cassie biped is an approximately one meter tall walking robot designed and manufactured by Agility Robotics. 
The design of the robot encompasses the physical attributes of the spring loaded inverted pendulum (SLIP) model dynamics, with the primary characteristic being a pair of light-weight legs with a heavy torso so that the system is approximated by a point-mass with virtual springy legs. 
On Cassie, a \textit{compliant multi-link mechanism} is used to transfer power from higher to lower limbs without allocating the actuators' weight onto the lower limbs, and effectively acts as a pair of springy legs through which the knee motors effectively drive the leg length. 
The sensing on the robot is entirely proprioceptive and includes an IMU, torque sensing, and absolute encoders. 
There are $10$ brushless DC motors controlling the joints through $8$ low-friction cycloidal gearboxes and $2$ harmonic gearboxes. 


\newsec{Generalized Coordinates.}
Assuming $\WorldFrame$ be a fixed world frame and $\BaseFrame$ be a body frame attached to the pelvis of the robot located at the center of the hip, then the Cartesian position $\BasePos = (\BasePos^x, \BasePos^y, \BasePos^z) \in \R^3$ and the orientation $\BaseRot = (\BaseRot^x, \BaseRot^y, \BaseRot^z) \in SO(3)$ of $\BaseFrame$ with respect to $\WorldFrame$, composes the floating base coordinates. 

The configuration of Cassie, as illustrated in \figref{fig:cassie_configuration}, consists of two kinematic chains: left leg joints, $q_{lleg}= [q_{lhr}, q_{lhy}, q_{lhp}, q_{lk}, q_{ls}, q_{lt}, q_{lhs}, q_{lak}]^T$, and right leg joints, $q_{rleg}= [q_{rhr}, q_{rhy}, q_{rhp}, q_{rk}, q_-{rs}, q_{rt}, q_{rhs}, q_{rak}]^T$, given as the hip roll, hip yaw, hip pitch, knee pitch, shin spring, tarsus pitch, heel spring, and ankle pitch joints, respectively. 
The actuated joints are symmetric for both legs, shown in \figref{fig:cassie_configuration}, and correspond to $q_{hr}, q_{hy}, q_{hp}, q_{k}$, and $q_{ak}$. As will be shown in the robot dynamics, there are also four passive compliant springs at the $q_{s}$ and $q_{hs}$ joints on each leg.
The configuration space $\ConfigSpace$ is given in the generalized coordinates:
\begin{equation*}
  q = (\BasePos, \BaseRot, q_l) \in \ConfigSpace = \R^3 \times
  SO(3) \times \BodyConfigSpace,
\end{equation*}
where $q_l$ is the coordinates of body configuration space $\BodyConfigSpace$ determined by $q_l = (\q_{lleg}, \q_{rleg}) \in \mathcal{Q}_l$.

\subsection{Continuous Dynamics:}
We model legged robots as a tree structure composed of rigid links . As legged locomotion inherently involves intermittent sequences of rigid contacts with the environment, it is common practice to construct a floating-base Euler-Lagrange model of the robot dynamics:
\begin{align}
    \label{eq:eom}
    D(q) \ddot{q} + H(q,\dot{q}) &= B u + J_c(q)^T \lambda_c, 
\end{align}
where $B$ is the actuation matrix with gear reductions as its entries, $u \in U \subset \R^m$ is the control input, the Jacobian matrix of the holonomic constraint is $J_c(q)=\partial \eta_v/\partial q$ with its corresponding constraint wrenches $\lambda_c\in\R^{m_{\eta}}$. The mass-inertia matrix, $D(q) := D^R(q) + D_m$, includes the nominal inertia matrix for the rigid linkages, $D^R(q)$, and the reflected motor inertia matrix, $D_m$. Finally, the vector $H(q,\dot{q}) := C(q,\dot{q}) \dot{q} + G(q) - \kappa(q,\dot{q})$ contains the Coriolis matrix, $C(q,\dot{q})$,  gravity vector, $g(q)$, and spring forces $\kappa(q,\dot{q})$.  
%

In the order of the coordinates defined previously for Cassie in the robot configuration, the actuated joint reflected inertia are $\mathcal{I}_{m,l} = [1.435046, 1.435046, 1.435046, 1.44662, 1.44662]$ for each leg. 
The springs are modeled with a linear torsional stiffness, $k_{s} = 2,300$ and $k_{hs} = 2,000$ N/m, and damping, $b_{s} = 4.4$ and $b_{hs} = 4$ N/m/s, forming a vector of torsional generalized forces at the spring pivots:
\begin{align*}
    \kappa(q,\dot{q}) = [\mathbf{0}_{1\times 10}, \ &
                    k_s q_\mathrm{ls} + b_s \dot{q}_\mathrm{ls}, \ 0, \ k_\mathrm{hs} q_\mathrm{lhs} + b_\mathrm{hs} \dot{q}_\mathrm{lhs}, \\
                    &\hspace{-7mm} \mathbf{0}_{1\times 5}, \ k_s q_\mathrm{rs} + b_s \dot{q}_\mathrm{rs}, \ 0,\ k_\mathrm{hs} q_\mathrm{rhs} + b_\mathrm{hs} \dot{q}_\mathrm{rhs}, \ 0]^T,
\end{align*}
where $\mathbf{0}_{(\cdot) \times (\cdot)}$ is matrix of all zeros.

\subsection{Holonomic Constraints} 
Two types of holonomic constraints are commonly considered for legged robotic systems, external contact constraints depending on the current configuration of the robot and it's interactions with the world, and internal kinematic constraints resulting from the robot geometry. Both types of constraints are enforced in the same manner, by prescribing closure constraint on the kinematics, $\eta(q) = \mathrm{constant}$. Differentiating $\eta(q)$ once, we obtain a kinematic constraint on velocity:
\begin{align*}
    0 = \underbrace{\frac{\partial \eta_c(q)}{\partial q}}_{J_c(q)} \dot{q}.
\end{align*}
Differentiating once more yields an acceleration constraint:
\begin{align}
    0 = J_c(q) \ddot{q} + \underbrace{\frac{\partial}{\partial q} \left( \frac{\partial J_c(q)}{\partial q} \dot{q} \right)}_{\dot{J}_c(q,\dot{q})} \dot{q}. \label{eq:hol_accel}
\end{align}
The enforcement of this equality constraint gives rise to the corresponding force terms, $\lambda_c$, in the equations of motion \eqref{eq:eom}.

\begin{figure}[t!]
\centering
	\includegraphics[width= 1\columnwidth]{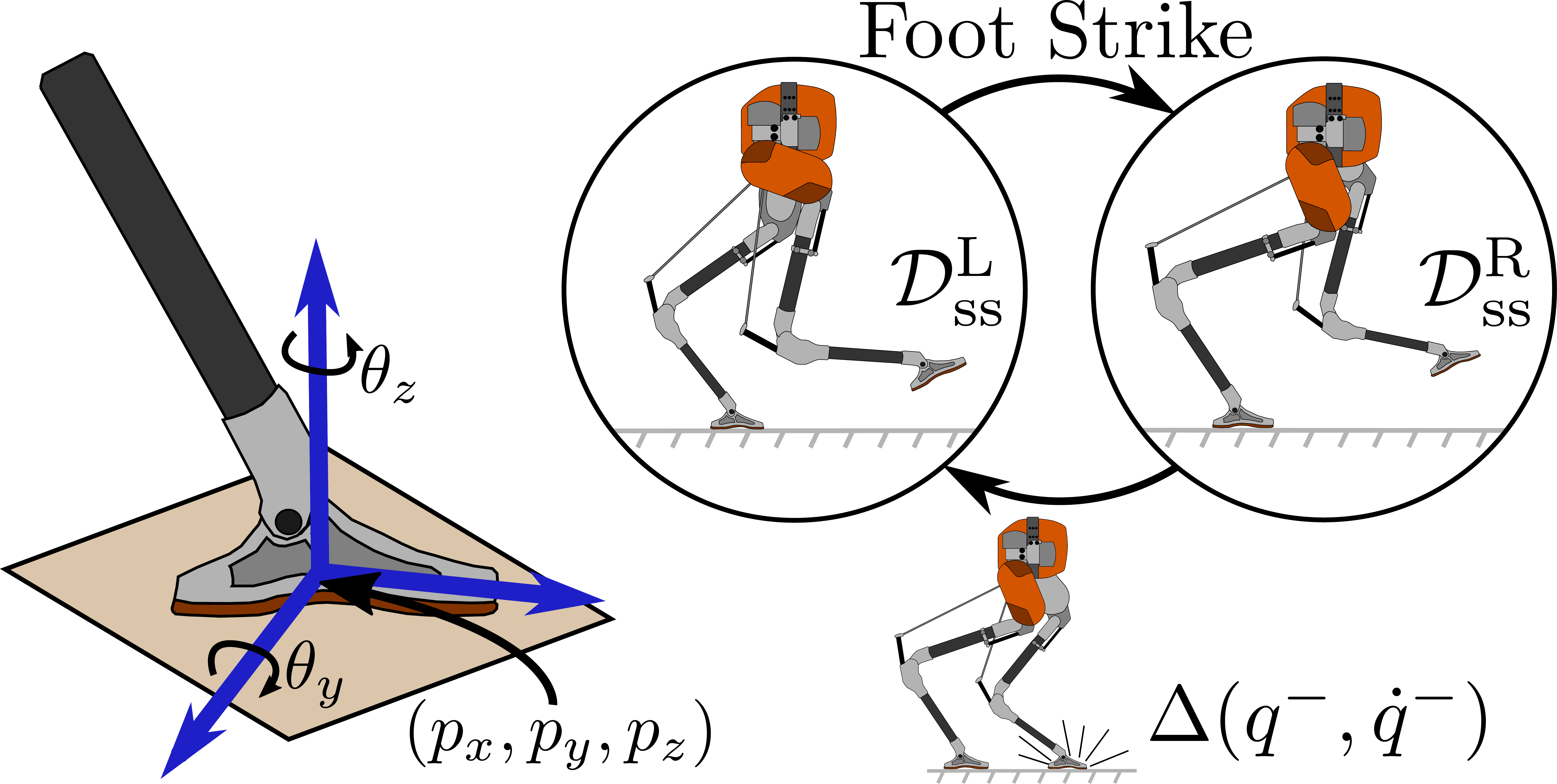}
	\caption{The directed graph of walking dynamics, where we view walking on Cassie as consisting of one domain with a compliant stance leg, and rigidly stiff swing leg. }
	\label{fig:direct}
\end{figure}

\newsec{Contact Constraints.} 
Contacts on the Cassie robot are enforced through a holonomic constraint on the stance foot's position and orientation, $\eta_{st}(q)$. 
Because the width of the feet on Cassie is negligible we enforce contact as a line, illustrated in \figref{fig:direct}, forming the $5$-DOF constraint:
\begin{align}
\label{eq:hol_contact}
    \eta_{st}(q)^{T} := \left[ p_{st}^x, p_{st}^y, p_{st}^z, \phi_{st}^y, \phi_{st}^z \right]^T,
\end{align}
where the first three components are the Cartesian position of the foot center and the last two correspond to the the foot pitch and yaw. The vertical Cartesian force, $\lambda_{st}^z$, is a normal force and thus is unilateral (i.e. $\lambda_{st}^z \geq 0$). Additionally, the tangential forces $\lambda_{st}^x$, $\lambda_{st}^y$ must satisfy friction models to remain physically feasible. Ideally, a classical \textit{Amontons-Coulomb model} of (dry) friction is used to avoid slippage and is represented as a \textit{friction cone}. For a friction coefficient $\mu$ and a surface normal, the space of valid reaction forces is:
\begin{align}
    \mathcal{C} = \left\{ \left. ( \lambda_{st}^x, \lambda_{st}^y, \lambda_{st}^z ) \in \R^3 \right| \sqrt{(\lambda_{st}^x)^2 + (\lambda_{st}^y)^2} \leq \mu \lambda_{st}^z \right\}. \label{eq:cone_friction}
\end{align}
However, this constraint is nonlinear, and cannot be implemented as a linear constraint. 
An alternative solution is to use a \textit{pyramidal friction cone} approximation \cite{grizzle2014models}:
\begin{align}
    \mathcal{P} = \left\{ \left. ( \lambda_x, \lambda_y, \lambda_z )\in \R^3 \right| |\lambda_x|, |\lambda_y| \leq \frac{\mu}{\sqrt{2}} \lambda_z \right\}. \label{eq:pyramid_friction}
\end{align}
This is a more conservative model than the friction cone, but is advantageous in that it is a linear inequality constraint.
Additionally, the moment associated with the foot pitch $\lambda_{st}^{my}$ can produce rotation of the foot over the forward edge if it is too large. It has been shown that due to the unilateral nature of contact this moment is limited by:
\begin{align}
\label{eq:footroll}
      -\frac{l}{2} \lambda_{st}^z <  &\lambda_{st}^{my} < \frac{l}{2} \lambda_{st}^z,
\end{align}
where $l$ is the length of the foot from heel to toe \cite{vucobratovic1990biped}.

\newsec{Kinematic Loop Constraints.} 
It is common practice to model bipedal robots as serial branched-tree structures. 
However, on Cassie, a compliant multi-bar mechanism forms a kinematic loop within the leg structure. 
When a mechanism has a kinematic loop, this is often managed by cutting the loop at one of the joints and enforcing a holonomic constraint at the connection to form the closed-chain manipulator. 
In the Cassie leg, the heel spring is attached to the rear of the tarsus linkage, with its end constrained via a pushrod affixed to the hip pitch linkage. For this work, we assume that pushrod attachment is a virtual holonomic distance constraint applied between the hip and heel spring connectors as:
\begin{align}
\label{eq:hol_ach}
    \eta_{ach}(q_l) := d(q_l) - 0.5012 = 0, 
\end{align}
where the attachment distance $d(q_l)\in\mathbb{R}$ is obtained via the forward kinematics between connectors at the hip and heel spring. This kinematic loop is illustrated in \figref{fig:cassie_configuration}, where the pushrod is ``virtual'' in our model and enforced via \eqref{eq:hol_accel}. We also assume that when a leg is in swing that the springs on that leg are rigidly fixed:
\begin{align}
    \eta_{sw}(q)^T := [q_s, q_{hs}]^T = 0. \label{eq:swing_rigid_constraint}
\end{align}
This simplifies both the optimization and control implementations \cite{sreenath2013embedding} that will be introduced in later sections, and makes the dynamics less numerically stiff.

\subsection{Hybrid Locomotion Model} \label{sec:hybrid_model}
Having already described the configuration of the robot, the coupled equations of motion obtained from \eqref{eq:eom} and \eqref{eq:hol_accel} can also be expressed as the nonlinear affine control system: \cite{grizzle2014models}:
\begin{equation}
    \dot{x} = f(x) + g(x)u,  \quad \mathrm{for} \quad x = (q^T,\dot{q}^T)^T. \label{eq:eom_nonlinear}
\end{equation}
While this ODE can describe the continuous dynamics of the walking robot, 
bipedal walking gaits consist of one or more different continuous
phases followed by discrete events that transition from one
phase to another. This motivates the use of a hybrid system formulation 
with a specific ordering of phases. Periodic robotic walking can then be understood as a 
directed cycle with a sequence of continuous
domains (continuous dynamics) and edges (changes in contacts).


In this work, we structure the dynamics of walking on Cassie in a hybrid fashion. 
The walking consists of two single support domains, $\mathcal{D}_\mathrm{SS}^{\{ \mathrm{L, R} \}}$, associated with stance on the respective left (L) or right (R) foot. 
An associated directed cycle, $\DirectedGraph = (\Vertex, \Edge)$, can be specified for the walking:
\begin{equation}
\begin{aligned}
  \Vertex &= \{ss^\mathrm{R}, ss^\mathrm{L}\}, \\
  \Edge &= \{ss^\mathrm{R} \to ss^\mathrm{L}, ss^\mathrm{L} \to ss^\mathrm{R}\},
\end{aligned}
  \label{eq:vertexedge-singledomain_cassie}
\end{equation}
where each vertex, $v \in V$, represents a continuous domain and each edge, $e\in E$, corresponds to a transition between these domains, as shown in \figref{fig:direct}. 
Specifically, walking on Cassie in this section is considered as a period-two cycle of alternating \textit{single-support}, which are connected by the state dependent event of impact. This means that the \textit{double-support} domain here is \textit{instantaneous}. 
While the inclusion of a double-support domain is the most physically accurate representation of the locomotion on Cassie, it adds a significant amount of additional variables and cardinal nodes to the optimization problem and makes the feedback control approaches we will later derive unnecessarily complex for this study \cite{reher2019dynamic}. 
As previously stated, the walking is considered to be asymmetric, or period-two, meaning that in optimization and in control development we consider the right and left stance as distinct. 
This is done to allow for modeling of lateral walking gaits, which cannot be represented by a symmetric motion. 

The mathematical model of the hybrid system representation of locomotion we wish to design is defined based on the formal framework of hybrid systems. The hybrid control system of the biped is defined as the tuple \cite{ames2007geometric,grizzle2014models}:
\begin{equation}
	\HybridControlSystem = (\DirectedGraph,\Domain,\ControlInput,\Guard,\ResetMap,\emph{FG}). \label{eq:hybrid_control_system}
\end{equation}
\begin{itemize}
    \item $\DirectedGraph = \{\Vertex,\Edge\}$ is a \emph{directed cycle} specific to the desired walking behavior, with $\Vertex$ the set of vertices, $v_{\mathrm{s}},v_{\mathrm{t}} \in V$, and $\Edge$ the set of edges, $e = (v_{\mathrm{s}} \to v_{\mathrm{t}}) \in E$.
	\item $\Domain=\{\Domain_{v} \}_{v \in V}$ is the set of \emph{domains of admissibility} consisting of admissible states on which \eqref{eq:eom_nonlinear} evolves,
	\item $\mathcal{U} \subseteq \mathbb{R}$ is the set of \emph{admissible control inputs},
	\item $S \subset \mathcal{D}$ is a \emph{guard} (or switching surface) that are the states when the swing foot strikes the floor,
	\item  $\ResetMap = \{\ResetMap_{\ei}\}_{\ei \in  E}$ is the set of \emph{reset maps}, $\ResetMap_{\ei} : \Guard_{\ei} \subset \Domain_{v_{\mathrm{s}}} \to \Domain_{v_{\mathrm{t}}}$ from one domain to the next,
	\item $\emph{FG}$ is the \emph{nonlinear control system} associated with the dynamics as given in \eqref{eq:eom_nonlinear}. 
\end{itemize}
The transition from one single support domain to another occurs when the vertical position of the non-stance foot crosses zero. Therefore, the domain and guard are given by:
\begin{align*}
    \mathcal{D}_\mathrm{SS}^{\{ \mathrm{L, R} \}} &= \{ (q, \dot{q}, u) :  p_{nsf}^z(q) \geq 0, \lambda_{nsf}^z(q,\dot{q},u) = 0\}, \\
    S_{\{ \text{L}\rightarrow\text{R, } \text{R}\rightarrow\text{L} \}} &= \{ (q,\dot{q}) :  p_{nsf}^z(q) = 0, \dot{p}_{nsf}^z(q,\dot{q}) < 0 \},
\end{align*}
where $\lambda_{nsf}^z(q,\dot{q},u)$ is the vertical ground reaction force of the swing foot and $p_{nsf}^z(q)$ is the vertical position of the center of the swing foot from the ground. 
An impact occurs when the swing foot touches the ground, modeled here as an inelastic contact between two rigid bodies. 
The configurations of the robot are thus invariant through the impact and velocities will instantaneously jump \cite{Hurmuzlu1994Rigid}. The associated reset map, $\Delta$, is:
\begin{align}
  \Delta(q^-,\dot{q}^-) := \left[\begin{array}{c}
      q^+ \\
      \dot{q}^+
    \end{array}
  \right] = \left[\begin{array}{c}
      \mathcal{R}(q^-) \\
      \frac{\partial \mathcal{R}(q^-)}{\partial q} \Delta^{\dot{q}}(q^-) \dot{q}^-
    \end{array}
  \right], \label{eq:impact_map}
\end{align}
where $\mathcal{R}(q)$ is a relabeling function, $q^-$ and $q^+$ denote the pre and post-impact configurations, and $\Delta^{\dot{q}}(q)$ is obtained from the plastic impact equation \cite{grizzle2014models}:
\begin{align}
\label{eq:vel_impact}
  \Delta^{\dot{q}}(q^-) = I - D^{-1} J_c^T(J_c
  D^{-1} J_c^T)^{-1} J_c.
\end{align}

\newsec{Reset Map Definition.}
Because the legs are compliant, they may not necessarily leave the ground with the springs at their neutral angle. For the majority of the walking that we consider, the swing leg is assumed to be sufficiently rigid to model the springs as a holonomic constraint while in the air. Thus, we must define a spring reset function, which can be applied as part of our relabeling matrix to zero the springs \cite{sreenath2013embedding}. 

While simply resetting the spring values to zero is sufficient for the shin and heel springs, we must solve for a nontrivial value on the tarsus. An inverse kinematics problem can then be used to solve the multi-bar (zero spring deflection) closure constraint given by \eqref{eq:hol_ach}:
\begin{align*}
    \bar{\gamma}_{\text{tar}}(q) &:= 0.028794 + 0.118906 \cos(q_\mathrm{k}) - 0.112216\cos(q_\mathrm{t}) \notag \\
                           & - 0.0280613 \cos(q_\mathrm{k} + q_\mathrm{t}) - 0.0161784 \sin(q_\mathrm{k})  \notag \\
                           & - 0.0425142 \sin(q_\mathrm{t}) - 0.00647928 \sin(q_\mathrm{k} + q_\mathrm{t}) = 0 .
\end{align*}
The inverse kinematics solution for the neutral tarsus angle is then denoted $\bar{q}_{\mathrm{t}}(q_{\mathrm{k}}) := f_{\bar{\gamma}_{\mathrm{tar}}}(q_{\mathrm{k}})$. 
Using this, we can then solve for the for the post-impact tarsus joint, $q_t^+$, given the pre-impact values for $q_k^-$ and assuming $q_s^+\rightarrow0$, $q_{hs}^+\rightarrow0$. 

The walking on Cassie in simulation and in optimization is most generally represented as a period two walking cycle, meaning that the cycle repeats after the left and right legs have both been through a stance phase. 
Thus, for our walking model, the reset map $\mathcal{R}: \mathcal{Q} \rightarrow \mathcal{Q}$ becomes:
\begin{align*}
    \mathcal{R}(q^-) := \left( \mathcal{R}_{b}^T(q^-), \mathcal{R}_{l}^T(q^-), \mathcal{R}_{l}^T(q^-) \right)^T , 
\end{align*}
where $\mathcal{R}_{b}(q^-) := q_{\mathrm{b}}^-$ and $\mathcal{R}_{l}(q^-)$ is a nonlinear function which is mostly the identity mapping combined with the spring zeroing inverse kinematics applied at the tarsus:
\begin{align*}
    \mathcal{R}_{l}(q^-) := \left( q_{\mathrm{hr}}, q_{\mathrm{hy}}, q_{\mathrm{hp}}, q_{\mathrm{kp}}, 0, \bar{q}_{\mathrm{t}}(q_{\mathrm{k}}^-), 0, q_{\mathrm{tp}} \right)^T,
\end{align*}
and it can be seen that the zero entries correspond to the shin and heel spring indices.

\subsection{Motivating the Compliant Model}
\begin{figure}[b!]
\centering
	\includegraphics[width= 0.8\columnwidth]{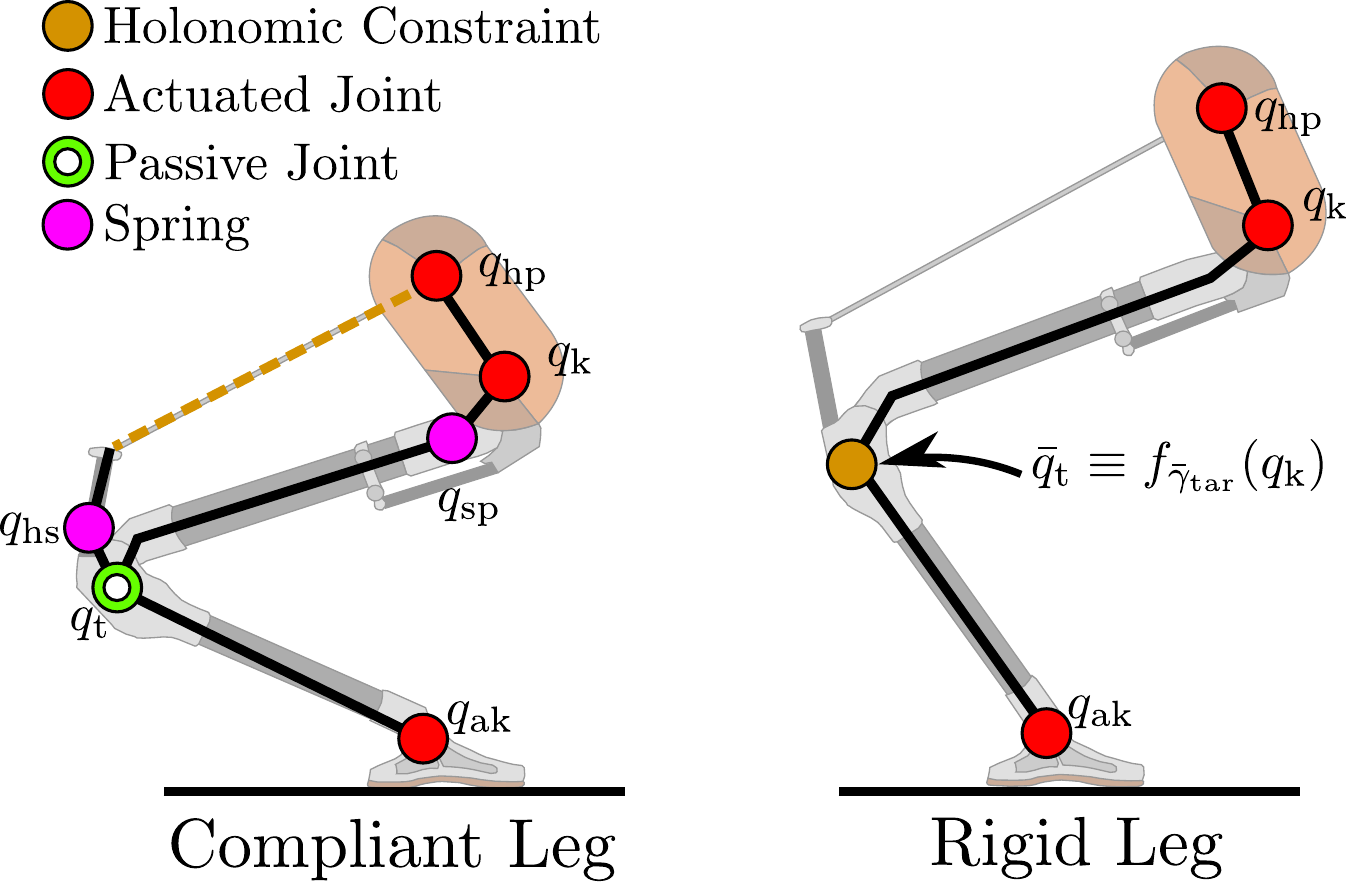}
	\caption{An illustration of the model differences between a compliant and rigid representation of the Cassie leg.} 
	\label{fig:compliant_and_rigid_definition}
\end{figure}
In this work, we will leverage the full compliance on Cassie in the context of dynamic locomotion.  
The leg of the Cassie robot, pictured in \figref{fig:cassie_configuration}, effectively forms a $6$-bar mechanism with $2$ fiberglass leaf springs. 
%
The spring action on Cassie acts along the sagittal plane of the leg mechanism in both the radial and tangential directions. 
This means if we control an output in the leg length and leg angle directions, there will an associated passive compliance along both directions, rather than just the axial length. 
A significant effort was made in this work to fully leverage this compliance by modeling the springs as underactuated coordinates which enter the zero dynamics, and thus considered explicitly in our control approach. 

In several existing works on Cassie \cite{gong2019feedback,hereid2018rapid}, it was shown how a rigid model of the robot could be used to generate stable walking behaviors. This model is shown on the right in \figref{fig:compliant_and_rigid_definition}, where the heel spring is removed, the shin spring is fixed at zero deflection, and the tarsus angle is purely a function of the knee angle as a result of a holonomic constraint which is imposed to close the undeflected four-bar linkage. 
One of the primary reasons that a more constrained model may be considered is that it requires fewer degrees of freedom and the compliant mechanism not only increases local stiffness of the nonlinear dynamics, but also induces model uncertainties for the springy joints. 
Despite these initial difficulties that spring dynamics may bring, the leg mechanism clearly has nontrivial compliance when in contact with the world. 

In this section we will discuss the advantages that using a compliant leg model may provide. 
The differences between the rigid and compliant models can be briefly summarized as:
\begin{enumerate}
    \item[-] \textbf{Rigid model:} assumes all four leaf springs are rigid linkages, which yields kinematic approximations as the geometry relation $\eta_{\mathrm{rigid}}(q) := \theta_\mathrm{k} - \theta_\mathrm{t} - 13^\circ \equiv 0$ for the leg. The constrained model with no contact is $16$ DOF.
    \item[-] \textbf{Compliant model:} instead treats the rotational joint of the leaf spring linkage as a torsional joint, with stiffness and damping effects. In addition, the distance between the hip and end of the heel spring remains a constant (as shown by the dash line in \figref{fig:compliant_and_rigid_definition}). This geometry relation can be described as a closure constraint: $\eta_{\mathrm{{ach}}} (q) \equiv 0$. The full compliant model with no contact is $22$ DOF.
\end{enumerate}

\newsec{A Comparison in Optimization and Simulation.}
\begin{figure}[t!]
\centering
	\includegraphics[width= 0.98\columnwidth]{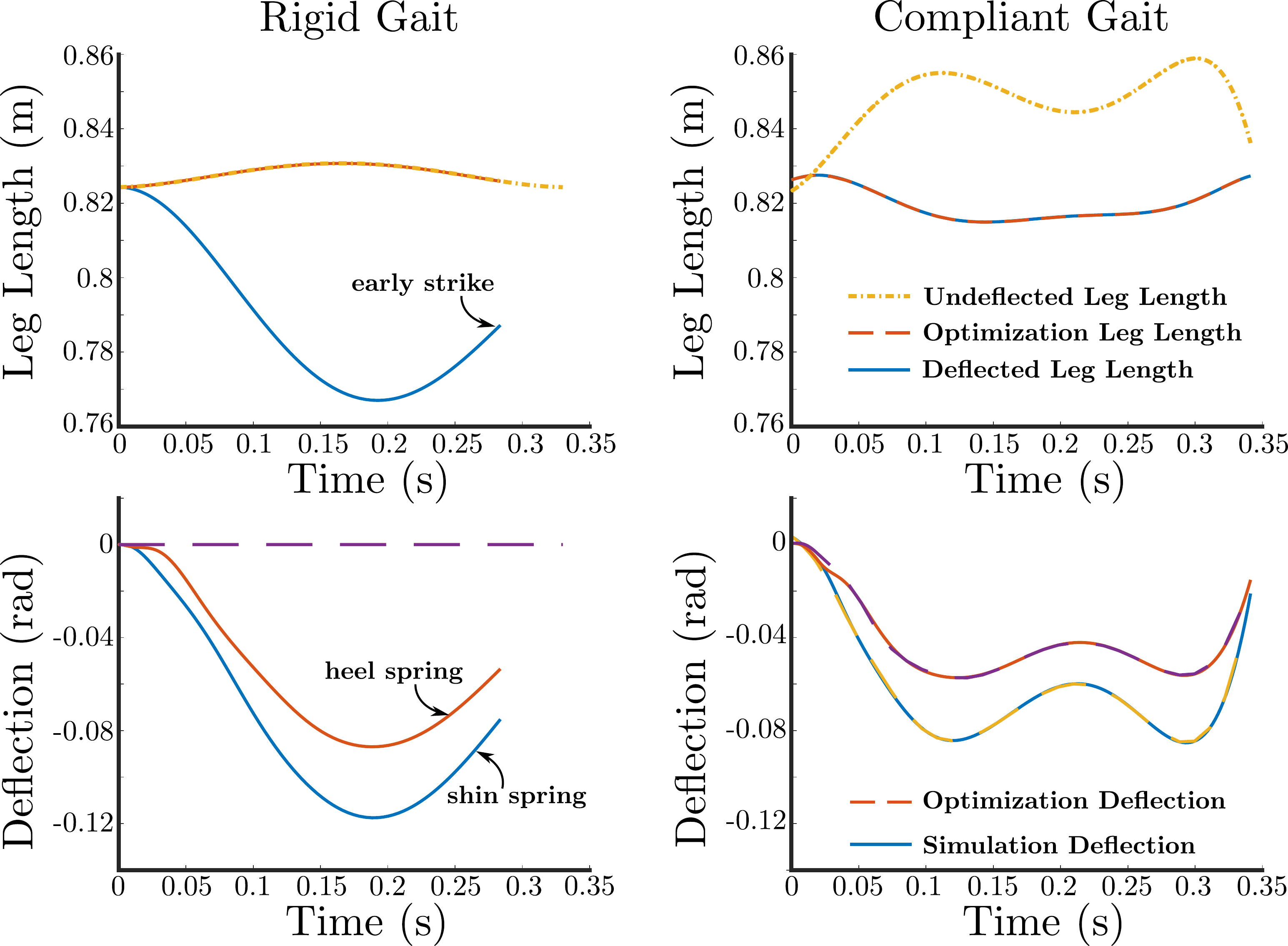}
	\caption{A comparison of the rigid model and compliant model for Cassie implemented in simulation. On the left, the rigid gait has not anticipated passive compliance, and thus drops and strikes the ground early. On the right the compliant motion has a plan for the shin and heel springs, meaning the neutral leg length output offsets to accommodate leg deflections.}
	\label{fig:spring_full_vs_rigid}
	\vspace{-3mm}
\end{figure}
In previous work, which \secref{sec:trajectory_hzd} will build upon, a trajectory optimization was used to find a single gait for both the compliant and rigid models of Cassie \cite{reher2019dynamic}. 
The main advantage of the compliant model is illustrated in \figref{fig:spring_full_vs_rigid}, where the rigid and compliant models were simulated. 
Because the rigid model has no planning for the passive degrees of freedom on the true system, the leg length sinks and causes an early strike. 
On the right, the difference between the neutral leg length (corresponding to the outputs in \secref{sec:trajectory_hzd}) anticipates this deflection. 

\newsec{Motivating the Compliant Gait Library.}
In \secref{sec:trajectory_hzd} we develop a gait library for compliant walking on Cassie at a variety of speeds. This not only provides a set of outputs which have planned for the passive compliance, such as the one shown in \figref{fig:spring_full_vs_rigid}, but it can also provide additional information useful in control design. 
Rather than simply track the output polynomials purely through a model-free PD control law, we would ideally have some feedforward information on how the dynamics should evolve through time. 
In other work on compliant HZD walking \cite{sreenath2013embedding} simply adding a feedforward torque into the control law provided sufficient torque for improved tracking. 
However, many bipedal implementations outside of HZD rely on inverse dynamics \cite{mistry2010inverse} and thus some parameterization of the generalized accelerations, $\ddot{q}$ \cite{feng2015optimization}. 
In this work, a model-based control approach will be introduced which uses a QP to track gaits in a pointwise-optimal fashion. 
The implementation of this controller in \secref{sec:implementation} uses regularization terms for the torque, forces, and accelerations.

\begin{figure}[t]
\centering
	\includegraphics[width= 0.92\columnwidth]{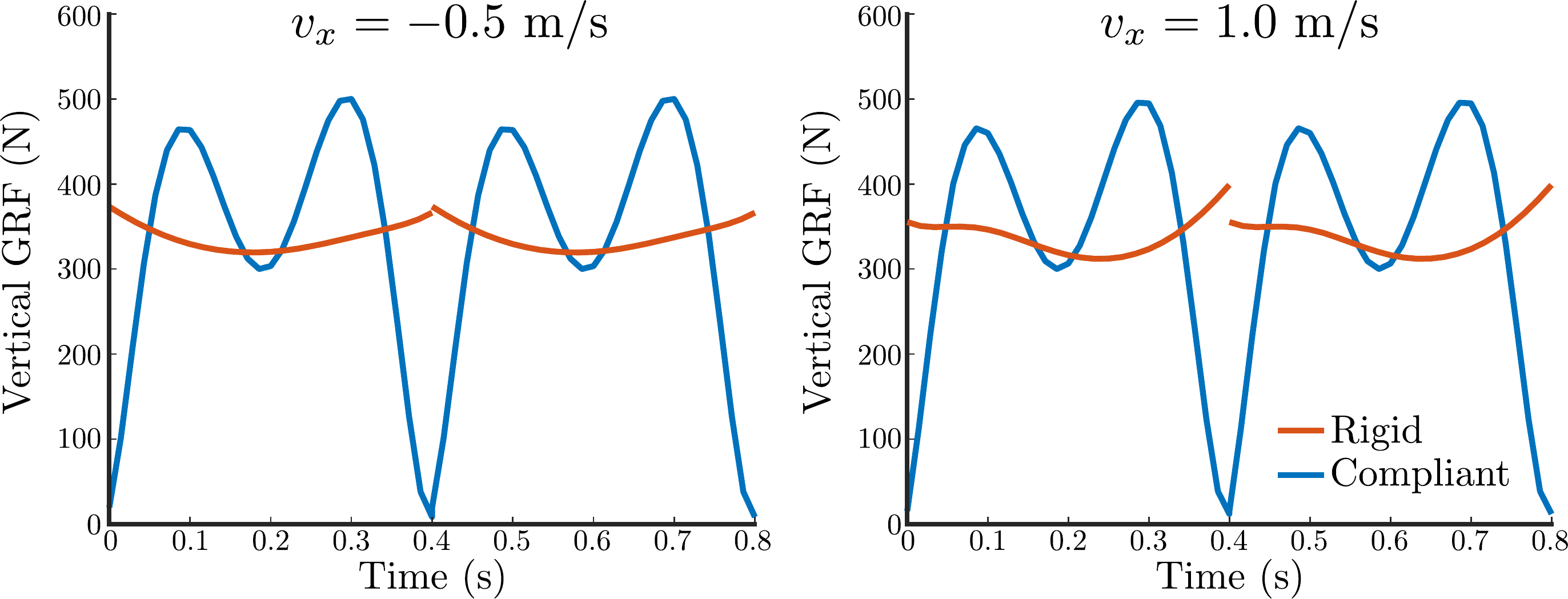}
	\includegraphics[width= 0.92\columnwidth]{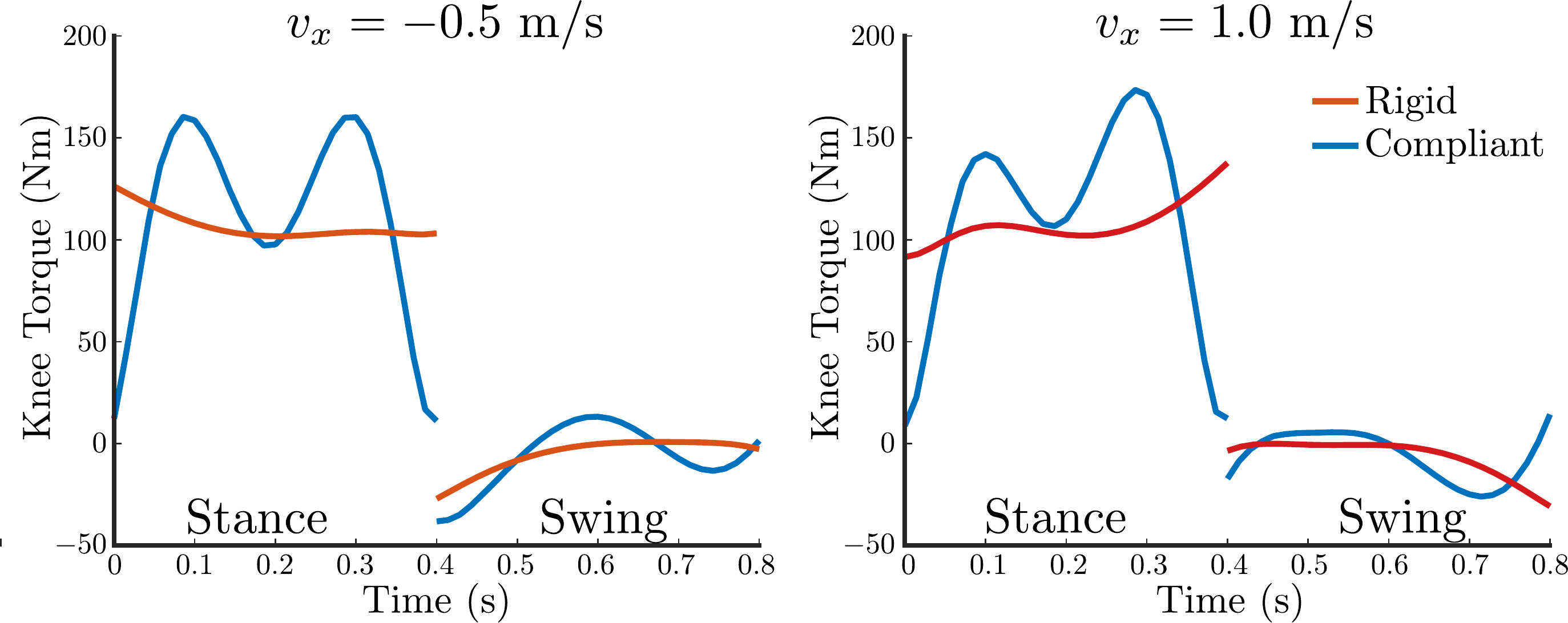} 
	\caption{(Top) The vertical ground reaction forces for compliant and rigid walking. (Bottom) Torque at the knee joint compared for the rigid and compliant models of Cassie. Because the knee directly corresponds to the leg length output, the emergent torque is very similar to the vertical force.}
	\label{fig:rigid_compliant_force}
\end{figure}

With the aim of developing feedforward and regularization terms for control development, we then like to investigate some of the characteristics of the compliant motion library found in \secref{sec:trajectory_hzd}, and compare them to a rigid collection of gaits. 
The implementation of the optimization for the compliant gait library is outlined in \secref{sec:trajectory_hzd}, where here we have imposed identical constraints and an identical cost for a rigid optimization. 
The SLIP model is an emergent behavior of the contact forces shown in \figref{fig:rigid_compliant_force}, where we have plotted the ground reaction forces for a gait walking, where one can see the ``double-hump'' force profile \cite{blickhan1989spring}. 
In comparison, the rigid model has an almost constant vertical force, meaning that this profile is not an accurate representation of the compliant leg dynamics. 
One of the most important characteristics with regards to implementation is smooth torque profiles. 
On Cassie, the highest torque joint is the knee pitch. 
If we observe the knee torque in \figref{fig:rigid_compliant_force}, we can see a profile similar to the vertical contact force, with smooth profiles and much smaller discontinuities at impact for the compliant model.

\begin{figure*}[t]%
	\centering
	\includegraphics[width=0.83\textwidth]{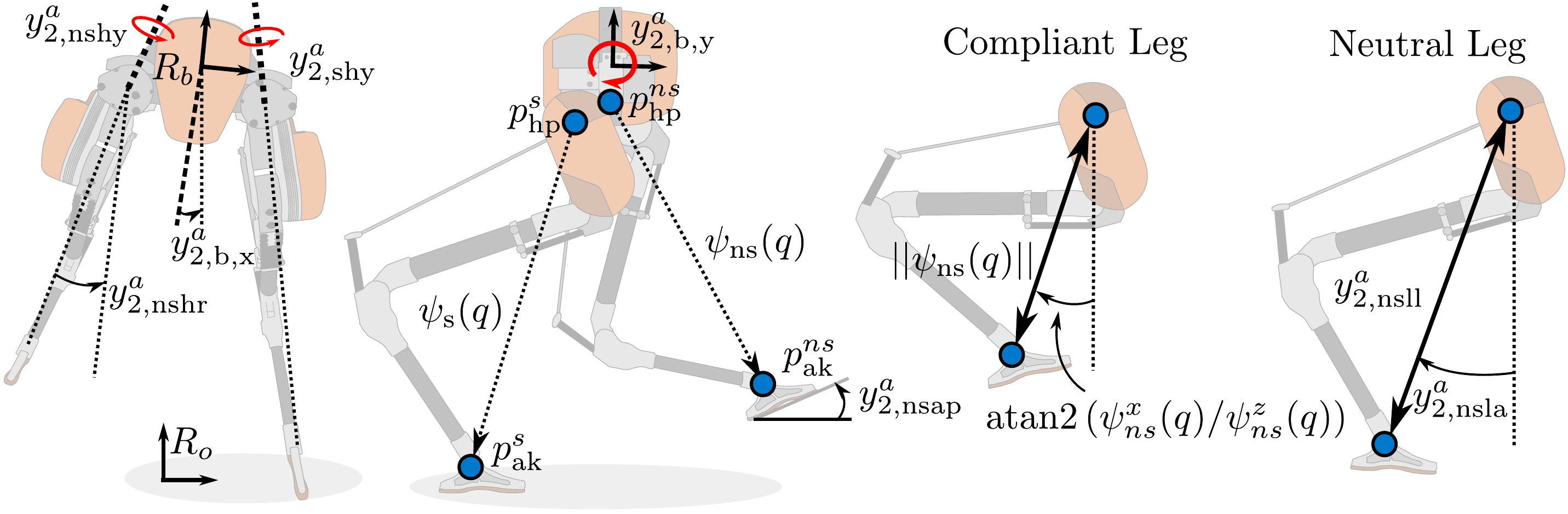}%
	\caption{A visualization of the outputs which can be selected for Cassie in this section. On the right is an illustration of the difference between the actual leg length and angle versus the neutral leg configuration, for which the neutral (meaning undeflected) positions are used as actual outputs.}%
	\label{fig:cassie_onedomain_output_definitions}%
	\vspace{-3mm}
\end{figure*}

\section{Hybrid Zero Dynamics Gait Planning} \label{sec:trajectory_hzd}
This section details the trajectory optimization used to design a collection of walking trajectories.
While each optimization will determine one stable orbit, it has been shown that one can expand the range of motions a robot can perform through systematic optimization to build parameterized gait libraries \cite{da20162d}. Reinforcement learning has also been used to handle transitions for different speeds and slopes \cite{da2017supervised}. 
The method presented in this work specifically seeks to obtain compliant behaviors which are representative of the physical system, and permit locomotion in both the sagittal and coronal directions using a gait library.

\subsection{Virtual Constraints and Feedback Linearization} \label{sec:fbl}
Analogous to holonomic constraints, virtual constraints are defined as a set of functions that regulate the motion of the robot with a desired behavior \cite{westervelt2018feedback}. 
The term ``virtual'' comes from the fact that these constraints are enforced through feedback controllers instead of through physical constraints. 
The primary idea is to design a controller $u(t, x, \alpha)$ to regulate:
\begin{equation}
\label{eq:outputs}
    y(\tau, x) := y^a(x) - y^d(\tau,\alpha),
\end{equation}
where $y^a : X \rightarrow \R^m$ and $y^d : \R \times \R^a \rightarrow \R^m$ are smooth functions encoding the desired behavior in a given domain. A $6$th-order B\'ezier polynomial is chosen for the desired outputs, for which $\alpha$ is a matrix of real coefficients.  

In the case of Cassie, we select nine actual outputs, $y^a(x)$, with (vector) relative degree $2$ \cite{sastry2013nonlinear}:
\begin{align*}
    y_{2,\mathrm{b},x}^a   &= \phi^x               &\text{(pelvis roll)}       \\
    y_{2,\mathrm{b},y}^a   &= \phi^y               &\text{(pelvis pitch)}      \\
    y_{2,\mathrm{sll}}^a   &= ||\psi_{s}(\bar{q}_l) ||_{2}     &\text{(stance leg length)} \\
    y_{2,\mathrm{nsll}}^a  &= ||\psi_{ns}(\bar{q}_l)||_{2}    &\text{(swing leg length)}   \\
    y_{2,\mathrm{nsla}}^a  &= \textrm{atan2}\left(\psi_{ns}^x(\bar{q}_l)/\psi_{ns}^z(\bar{q}_l)\right) &\text{(swing leg pitch)}\\
    y_{2,\mathrm{nshr}}^a  &= q_{\mathrm{nshr}}                  &\text{(swing hip roll)}   \\
    y_{2,\mathrm{shy}}^a   &= q_{\mathrm{shy}}                   &\text{(stance hip yaw)}   \\
    y_{2,\mathrm{nshy}}^a  &= q_{\mathrm{nshy}}                  &\text{(swing hip yaw)}    \\
    y_{2,\mathrm{nsap}}^a  &= \phi^y(q_{\mathrm{b}}, \bar{q}_l)  &\text{(swing foot pitch)}
\end{align*}
where $\phi^y(\theta_{\textrm{tp}})$ is the swing foot pitch angle and,
\begin{equation}
    \psi(q) = p_{\mathrm{hp}}(q) - p_{\mathrm{ak}}(q),
\end{equation}
is the expression for the distance between the hip pitch and ankle pitch joints.  In addition, we leave the stance foot passive. 
Because we use the neutral leg length, we can remove the tarsus and spring coordinates from the expressions using a substitution that gives the configuration of a ``neutral'' leg  
$\bar{q}_l \in \{\mathcal{Q}_l \ | \ q_{\mathrm{sp}}=0, q_{\mathrm{hs}}=0, q_{\mathrm{t}} = 13^o - q_{\mathrm{k}} \}$, leaving simplified expressions for the leg length and leg angle:
\begin{align*}
    y_{2,\mathrm{ll}}^a &= 0.727\sqrt{1.002 + \cos(q_\mathrm{k}) - 0.035\sin(q_\mathrm{k})}\\
    y_{2,\mathrm{nsla}}^a &= \textrm{atan2}\left(
    \frac{-0.053(\cos(q_{\mathrm{hp}}) + 9.971\sin(q_{\mathrm{hp}}) }{0.527(\cos(q_{\mathrm{hp}}) - 0.1\sin(q_{\mathrm{hp}})} \right. \dotsb  \notag\\
    &\hspace{5pt}\dotsb\left.
        \frac{+ 1.277\cos(q_{\mathrm{hp}} + q_k) + 9.382\sin(q_{\mathrm{hp}} + q_k))}{+ 0.941\cos(q_{\mathrm{hp}} + q_\mathrm{k})  - 0.128\sin(q_{\mathrm{hp}} + q_\mathrm{k}))}
    \right).
\end{align*}
The full geometry of the relevant expressions, along with the output definitions, are illustrated in \figref{fig:cassie_onedomain_output_definitions}. 
By formulating the outputs in this way, the passive dynamics of the system (and thus the zero dynamics) will contain the additional coordinates associated with the compliant elements \cite{sreenath2013embedding}. As a practical matter, this is also important as directly controlling the compliance in the leg is significantly more difficult \cite{ames2014quadratic}. 

\newsec{Feedback Linearization.}
In order to actually encode the walking behaviors described by the output polynomials, we must prescribe a feedback controller which can drive $y(\tau,x) \rightarrow 0$. Feedback linearization is a commonly used tool within the HZD community for this purpose, which transforms a nonlinear system into a linear one given a suitable change of variables and control input \cite{khalil2002nonlinear, isidori1997nonlinear}. 

Let us begin the derivation of our preliminary controller by considering the second derivative of our outputs:
\begin{align}
    \ddot{y}_2(q,\dot{q}) = \underbrace{\frac{\partial}{\partial q} \Big( \frac{\partial y_2}{\partial q} \dot{q}\Big) \dot{q} + \frac{\partial y_2}{\partial q} \Big[ -D^{-1} H \Big]}_{L_f^2 y_2} + \underbrace{\frac{\partial y_2}{\partial q} D^{-1} B }_{L_g L_f y_2} u, \notag
\end{align}
where $L_{f}$ and $L_{g}$ are the Lie derivatives with respect to the vector fields $f(x)$ and $g(x)$. For more concise representation terms are also grouped with a common notation:
\begin{align}
    \ddot{y}_2 &= \underbrace{\begin{bmatrix} \frac{\partial }{\partial q} \left(\frac{\partial y_2}{\partial q}\dot{q} \right) & \frac{\partial y_2}{\partial q} \end{bmatrix} f(x)}_{\mathbf{L}_f y} + \underbrace{\frac{\partial y_2}{\partial q}  g(x)}_{\mathcal{A}} u,
\end{align}
where $\mathcal{A}(x)$ is termed the \textit{decoupling matrix}, and is invertible. 
We can then prescribe the following control law:
\begin{align}
    u_{\text{IO}}(t,x) = \mathcal{A}^{-1}\Big(-\mathbf{L}_f y + \nu\Big) \hspace{3mm} \implies \hspace{3mm} \ddot{y} = \nu, \label{eq:FBL}
\end{align}
with an auxiliary control input $\nu$. 
Assuming that the preliminary feedback \eqref{eq:FBL} has been applied to \eqref{eq:eom_nonlinear}, we will render a linear system for the output dynamics with the specific choice of coordinates $\eta := (y_2^T, \dot{y}_2^T)^T$:
\begin{align}
    \dot{\eta} &= \begin{bmatrix} \dot{y}_2 \\ \ddot{y}_2 \end{bmatrix} = \underset{F}{\underbrace{\begin{bmatrix}  0 & \mathbf{I} \\ 0 & 0 \end{bmatrix}}} \eta + \underset{G}{\underbrace{\begin{bmatrix} 0 \\ \mathbf{I} \end{bmatrix}}} \nu. \label{eq:output_dynamics}
\end{align}
A valid choice of $\nu$ which stabilizes this linear system is: 
\begin{align}
    \nu = \ddot{y}_2  = - \frac{1}{\epsilon^2} K_P y_2 - \frac{1}{\epsilon} K_D \dot{y}_2 , \label{eq:IO_auxcontroller}
\end{align}
where $0 < \epsilon \leq 1$ is a tunable parameter, and $K_P$, $K_D > 0$ are control gains for the relative degree $2$ output error. 
This can be grouped into the closed-loop linear system:
\begin{align}
    \begin{bmatrix} \dot{y}_2 \\ \ddot{y}_2 \end{bmatrix} = 
    \underset{F_{\text{cl}}}{\underbrace{\begin{bmatrix} 0 & \mathbf{I} \\  -\frac{1}{\epsilon^2} K_P & - \frac{1}{\epsilon} K_D \end{bmatrix}}} \begin{bmatrix} y_2 \\ \dot{y}_2 \end{bmatrix}.
\end{align}
Since $F_{\text{cl}}$ is Hurwitz by definition (meaning that $\mathrm{Re}(\mathrm{eig}(F_{\text{cl}})) < 0$), the resulting linear dynamics is exponentially stable. In addition, the control parameter $\epsilon$ forces the system to converge at a rate governed by $\epsilon$. 

\begin{figure*}[t!]
	\centering
	\includegraphics[width=0.95\textwidth]{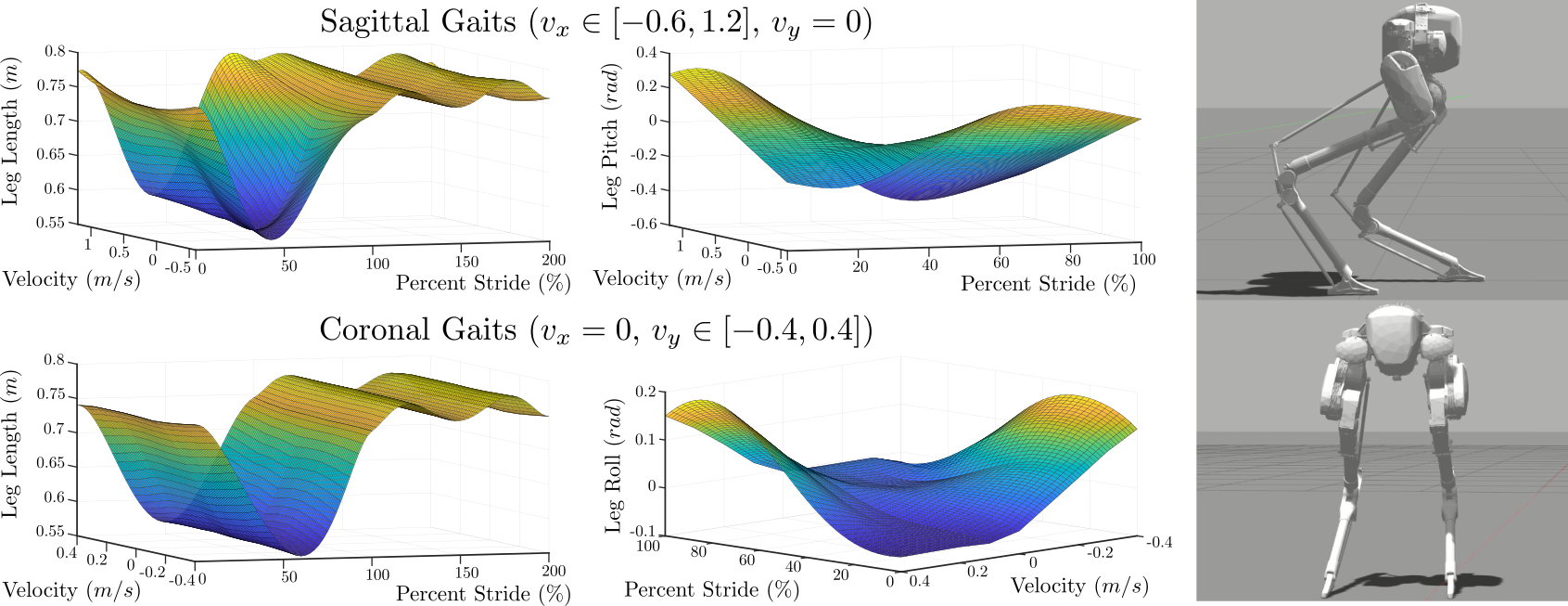}
	\caption{Contour plots of the swing leg length, leg angle, and leg roll outputs over the library speeds in the sagittal and coronal directions, showing the forward and reverse sweep of the leg as it tracks the motions. Also shown is the corresponding motion executed in a Gazebo simulation.}
	\label{fig:opt_output_continuum_feedforward}
\end{figure*}
%

\subsection{Hybrid Zero Dynamics} 
Using the control law in \eqref{eq:FBL}, we can apply $u_{\text{IO}}(t,x)$ with $\nu$ in \eqref{eq:IO_auxcontroller} to exponentially stabilize the linear dynamics \eqref{eq:output_dynamics}. Further, substitution of this controller into \eqref{eq:eom_nonlinear} yields the closed-loop dynamics: 
\begin{align}
\label{eq:eom_nonlinearcl}
    \dot{x} = f^\alpha(\tau,x,\alpha) = f(x) + g(x) u_{\text{IO}}(\tau,x,\alpha).
\end{align}
The hybrid system associated with the closed-loop dynamical system is described as \cite{hereid2018dynamic}:
\begin{align}
    \HybridSystem^\alpha \triangleq 
    \begin{cases} 
        \dot{x} = f^\alpha(x,\alpha) &\text{if} \hspace{5mm} x \in \Domain^\alpha \backslash \Guard^\alpha \\
        x^+= \ResetMap(x^-) &\text{if} \hspace{5mm} x^- \in \Guard^\alpha
    \end{cases} \label{eq:hybrid_closed_loop_dynamics}
\end{align}
where $f^\alpha$ is the dynamical system defined on $\Domain$ for the system \eqref{eq:eom_nonlinearcl}. 
Additionally, by driving the outputs to zero, the controller renders the 
\emph{zero dynamics} submanifold: 
\begin{align}
\label{eq:zerodyn}
  \ZD{}^\alpha = \{(\q,\dq) \in \Domain | y_{2}= \zm, L_f y_{2} = \zm\}
\end{align}
forward invariant and attractive \cite{isidori1997nonlinear}. 
Suppose then 
that there exists a local coordinate transformation $\Phi^z: \Domain \rightarrow \ZD{}$ and $\Phi^\eta : \Domain \rightarrow \Domain^\alpha$ so that $(\eta, z) = (\Phi^\eta(x), \Phi^z(x)) := \Phi(x)$. We can then write our controlled system in \textit{normal form} as: 
\begin{align}
    \dot{\eta} &= \bar{f}(\eta,z) + \bar{g}(\eta,z) u \label{eq:closed-loop-output-dynamics}\\
    \dot{z} &= \omega(\eta,z) \notag
\end{align}
with the restriction dynamics:
\begin{align*}
    \bar{f}(\eta(x), z(x)) = \begin{bmatrix}
        \dot{y}(x) \\
        \mathbf{L}_f y(x)
    \end{bmatrix}, \hspace{5mm}
    \bar{g}(\eta(x),z(x)) = \begin{bmatrix}
        0 \\ \mathcal{A}(x)
    \end{bmatrix}.
\end{align*} 
Thus, we can write the zero dynamics as the maximal dynamics compatible with the output equal to zero:
\begin{align}
    \dot{z} = \omega(0,z) := f|_{\mathcal{Z}}^\alpha(z). \label{eq:zero_dynamics_contdynamics}
\end{align}
Thus, the continuous dynamics \eqref{eq:eom_nonlinearcl} will evolve on $\mathcal{Z}^\alpha$. However, because \eqref{eq:zerodyn} has been designed without taking into account the hybrid transition maps \eqref{eq:impact_map}, it will not be impact invariant. In order to enforce impact invariance, the B\'ezier polynomials for the desired outputs can be shaped through the parameters $\alpha$. This can be interpreted as the condition \cite{westervelt2003hybrid}:
\begin{align}\label{eq:resetmapv}
 \Delta ( \mathcal{Z}^\alpha \cap S^\alpha ) \subset \mathcal{Z}^\alpha,
\end{align}
and will be imposed as a constraint on the pre and post-impact states through impact \eqref{eq:impact_map}. When \eqref{eq:resetmapv} is satisfied, we say that the system lies on the \textit{hybrid zero dynamics} (HZD) manifold. While the conditions derived here demonstrate HZD for a single-domain case, other work has shown that the multi-domain case follows in a similar fashion \cite{hereid20163d}.

The stability of hybrid systems 
is often determined by the existence and stability of periodic orbits. 
If the system \eqref{eq:eom_nonlinearcl} has HZD, then due to the hybrid invariance of $\mathcal{Z}$, there exits a stable hybrid periodic orbit, $\mathcal{O}|_{\mathcal{Z}} \subset \mathcal{Z}$, for the reduced order zero dynamics evolving on $\mathcal{Z}$, i.e., if we are evolving on the restriction dynamics of $f^\alpha|_{\mathcal{Z}}(z)$, then $\mathcal{O}|_{\mathcal{Z}}$ is a stable hybrid periodic orbit for the restricted dynamics in \eqref{eq:eom_nonlinearcl} \cite{westervelt2018feedback}. 
More concretely, let $\phi_t^{f^\alpha}|_{\mathcal{Z}}(z_0)$ be the (unique) solution to \eqref{eq:eom_nonlinearcl} at time $t\geq 0$ with initial condition $z_0$. 
For a point $z^*\in S$ we say that $\phi_t^{f^\alpha}|_{\mathcal{Z}}$ is hybrid periodic if there exists a $T>0$ such that $\phi_T^{f^\alpha}|_{\mathcal{Z}}(\Delta(z^*))=z^*$. 
Further, the stability of the resulting hybrid periodic orbit, $\mathcal{O}|_{\mathcal{Z}} = \{ \phi_t^{f^\alpha}|_{\mathcal{Z}}(\Delta(z^*)):0 \leq t \leq T \}$, 
can be found by analyzing the stability of the Poincar\'{e} map.

\subsection{Gait Optimization}
%
The problem of finding stable dynamic walking can now be transcribed to the nonlinear programming (NLP) problem of finding a fixed point $x^*$ and set of parameters $\alpha$ parameterizing the virtual constraints of \eqref{eq:outputs}. 
The optimization problem in this work is performed over one full step cycle, with a discrete impact \eqref{eq:impact_map} applied to the terminal state so that it satisfies the HZD condition of \eqref{eq:resetmapv}. It is also critical that the motions respect the limitations of the physical system such as the friction cone \eqref{eq:cone_friction}, foot rollover \eqref{eq:footroll}, and actuator limits. 
These constraints can be transcribed into an NLP and solved \cite{hereid2017frost}: 
\begin{align} 
	\mathbf{w}(\alpha)^* = &\underset{\mathbf{w}(\alpha)}{\mathrm{argmin}} \hspace{3mm} \mathcal{J}(\mathbf{w}(\alpha)) \tag{HZD Optimization}\label{eq:opteqs} \\
	\mathrm{s.t.} 		&\hspace{3mm}  \text{Closed\ loop\ dynamics: \eqnref{eq:eom_nonlinearcl}}  \notag \\
				  		&\hspace{3mm}  \text{HZD\ condition: \eqnref{eq:resetmapv}}  \notag\\
				  		&\hspace{3mm}  \text{Physical feasibility (e.g. \eqnref{eq:cone_friction})}   \notag
\end{align}
where $\mathbf{w}(\alpha)\in\mathbb{R}^{N_w}$, with $N_w$ being the total number of optimization variables and here we made the dependence on the parameters, $\alpha$, explicit. In order to minimize torque and to center the floating base orientation movement around the origin, the following cost function was minimized:
\begin{align}
    \mathcal{J}(\mathbf{w}) &:= \int_{t=0}^{tf} \left( c_u |u|^2 + c_\phi |\phi_b|^2 \right) dt,
    \label{eq:cost}
\end{align}
with $c_u = 0.0001$ and $c_\phi = (20,1,30)$.

\begin{table}[t]
\centering
\caption{Optimization constraints and parameters}
\label{table:optimization_constraints}
\def\arraystretch{1.1} 
    \begin{tabular}{ | m{15.4em} m{2.2cm} m{0.50cm} | } 
        \hline
        Step duration  &  $=0.4$ & sec \\ 
        \hline
        Average step velocity, $\bar{v}_{x,y}$ & $=v_{x,y}$ & m/s \\
        \hline
        Pelvis height, $p_z$  &  $\geq 0.80$ & m \\ 
        \hline
        Terminal spring deflection, $q_{\text{sp},\text{hsp}}(t_f)$    &  $=0$ & rad \\ 
        \hline
        Mid-step foot clearance, $p_{nsf}^z$  &  $\geq 0.14$ & m \\ 
        \hline
        Vertical impact velocity, $\dot{p}_{sw}^z$  &  $\in(-0.40, -0.10)$ & m/s \\
        \hline
        Step width, ${p}_{lf}^y - {p}_{lf}^y$  &  $\in(0.14, 0.35)$ & m \\ 
        \hline
        Swing foot pitch, $\phi^y(q)$ & $=0$ & rad \\
        \hline
        Friction cone, $\mu$    &  $< 0.6$ &  \\ 
        \hline
    \end{tabular}
    \vspace{-2mm}
\end{table}

Similar to \cite{xie2020learning,gong2019feedback}, we would like to design a variety of walking speeds for which the robot can operate. To accomplish this, a library of walking gaits at sagittal speeds of $v_x\in [ -0.6,1.2 ]$ m/s and coronal speeds of $v_y \in [-0.4, 0.4]$ m/s are generated in a grid of $0.1$ m/s intervals.
This resulted in $171$ individual optimization problems to be solved. 
Each optimization was then solved sequentially through the C-FROST interface \cite{hereid2017frost} on a laptop with an Intel Core i7-6820 HQ CPU @ $2.7$ GHz with $16$ GB RAM, and consisted of $8418$ variables with $4502$ equality and $5880$ inequality constraints. Using each gait as an initial guess to warm-start the next speed in the library, the average number of iterations per run was $199$ with an average total evaluation time of $263.8$ seconds.

\begin{figure}[t]
	\centering
	\includegraphics[width=0.9\columnwidth]{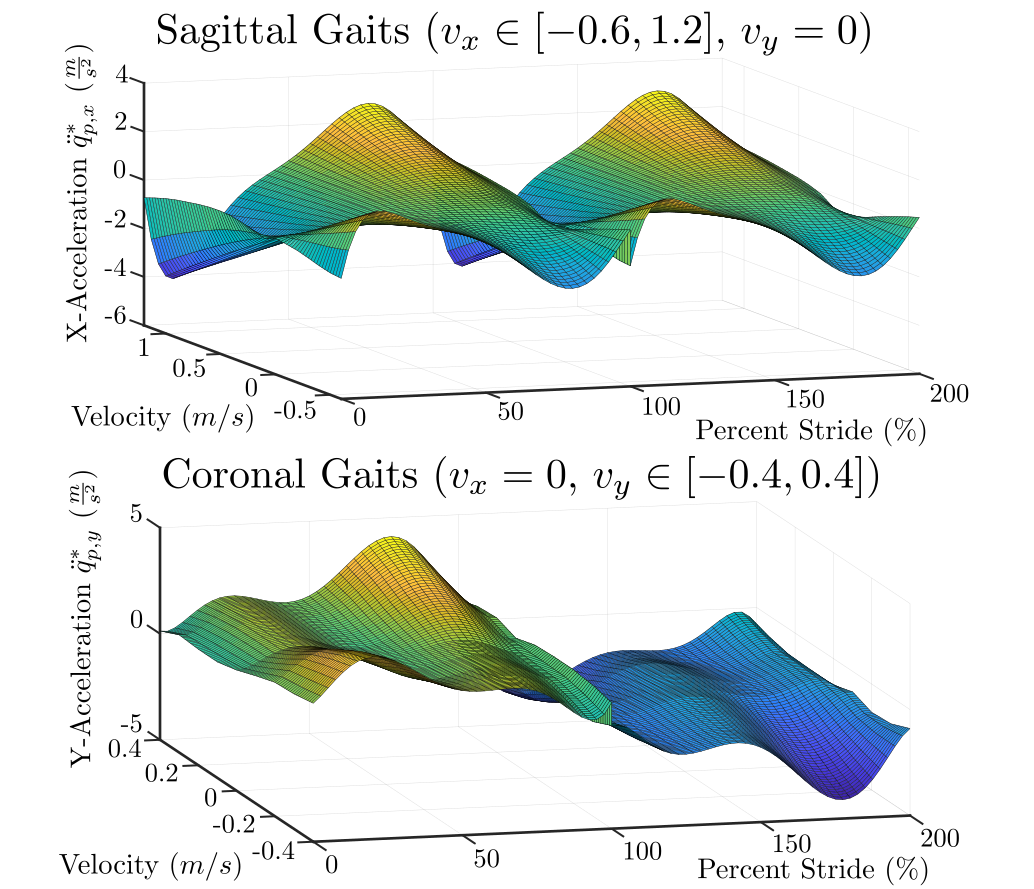}
	\caption{The contours of the floating base $x$ and $y$ accelerations which are obtained from the trajectory optimization problem. }
	\label{fig:ddq_plots}
\end{figure}

\begin{figure*}[t!]
	\centering
	\includegraphics[width=1\textwidth]{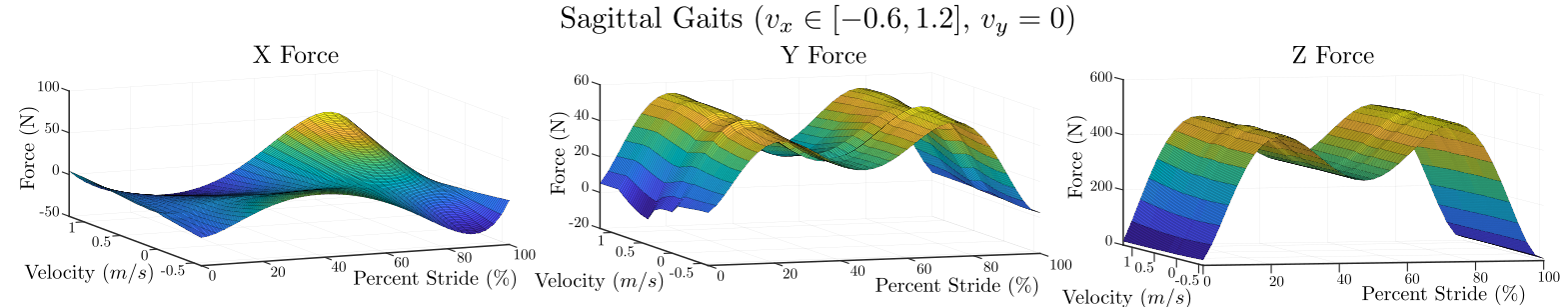}
	\includegraphics[width=1\textwidth]{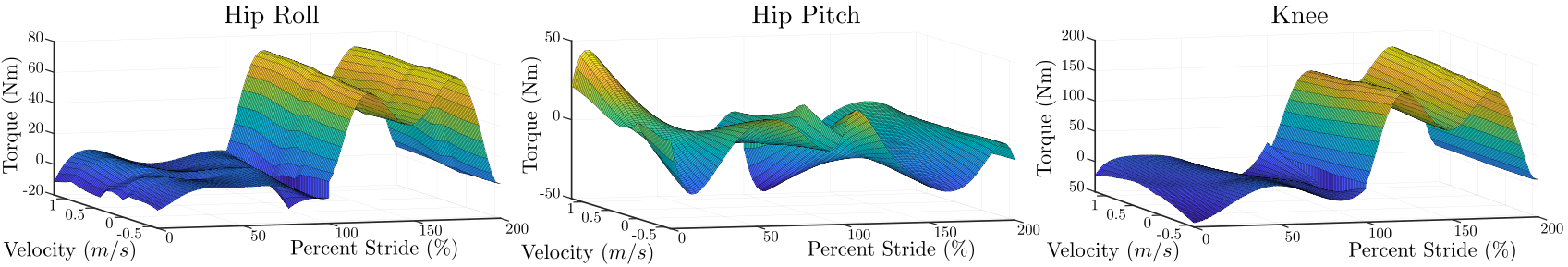}
	\caption{(Top) Contour plots showing the parameterized ground reaction forces for the right stance domain obtained from the optimization in the sagittal direction. (Bottom) Contour plots showing the parameterized torques for swing and then stance for walking on Cassie for sagittal walking speeds.}
	\label{fig:cassie_clf_forcereg}
\end{figure*}

\newsec{Gait Parameterization}
The controller implemented on hardware requires the feedback control objectives, defined by $y^d(\alpha,t)$, acceleration information $\ddot{q}^*$, and contact forces $\lambda_c^*$ from the optimal path to track the planned motions. The desired output parameters, $\alpha$, already concisely parameterize the feedback control, and can be placed in a large matrix for use with an interpolation routine. 
Several output polynomials are shown in \figref{fig:opt_output_continuum_feedforward}, where the leg length, leg pitch, and hip roll outputs are visualized over various walking speeds alongside simulation results. The leg length outputs shown in \figref{fig:opt_output_continuum_feedforward} demonstrate a SLIP-like ``double-hump'' shape \cite{blickhan1989spring} corresponding to the contact force profiles in \figref{fig:cassie_clf_forcereg}. 

Generalized accelerations $\ddot{q}^*$, torques $u^*$, and constraint forces $\lambda_c^*$ are extracted directly from the optimization variables, $\mathbf{w}$. To allow for easier implementation, regression is performed on each curve to obtain a $6$th order B\'ezier polynomial. They can then be stacked with the $\alpha$ parameters in the same bilinear interpolation routine for code efficiency. Plots of the accelerations for the floating base $x$ and $y$ coordinates are visualized in \figref{fig:ddq_plots}. Finally, the floating-base position  $p^*_{x,y}$ and velocity $v^*_{x,y}$ relative to the stance foot is also extracted. 


The bilinear interpolation routine is built on the assumption that each parameterized variable from the optimization is a rectangular matrix, $\alpha_{\square} \in \R^{M+1\times n_{\alpha_\square}}$, where $n_{\alpha_{\square}}$ is the dimension of the variable in question (i.e. $n_{\alpha_{\ddq}}=22$ for Cassie as $n=22$).
We first flatten each parameter matrix and then concatenate them into a single parameter array:
\begin{align}
    \beta^{i,j} := (\bar{\alpha}_y, \bar{\alpha}_p, \bar{\alpha}_v, \bar{\alpha}_{\ddot{q}}, \bar{\alpha}_u, \bar{\alpha}_{\lambda_c}).
\end{align}
Each flattened array is then organized into a matrix of arrays which is sorted and tagged with the corresponding velocity. 
In order to find the 
parameters for a given speed, we then search the matrix for the corresponding indices for the current $v_x$ and $v_y$. Next, we interpolate in the sagittal direction: 
\begin{align*}
    p(v_x,v_y^{j})   &= \frac{v_x^{i+1} - v_x}{v_x^{i+1} - v_x^{i}} \beta^{i, j} + \frac{v_x - v_x^{i}}{v_x^{i+1} - v_x^{i}} \beta^{i+1,j}, \\
    p(v_x,v_y^{j+1}) &= \frac{v_x^{i+1} - v_x}{v_x^{i+1} - v_x^{i}} \beta^{i, j+1} + \frac{v_x - v_x^{i}}{v_x^{i+1} - v_x^{i}} \beta^{i+1,j+1},
\end{align*}
where we can then also interpolate in the coronal direction to obtain the desired parameters:
\begin{align}
    p(v_x,v_y) &= \frac{v_y^{j+1} - v_y}{v_y^{j+1} - v_y^{j}} p(v_x,v_y^{j}) + \frac{v_y-v_y^{j}}{v_y^{j+1}-v_y^{j}} p(v_x,v_y^{j+1}). \label{eq:cassie_library_interp}
\end{align}
\section{Inverse Dynamics-Based Control Lyapunov Functions} \label{sec:id-clf}
In \secref{sec:fbl}, it was shown how feedback linearization could be used to render a linear system which could be stabilized via PD feedback. 
Instead, we would like to examine output tracking from a Lyapunov perspective. 
Control Lyapunov functions (CLFs), and specifically rapidly exponentially stabilizing control Lyapunov functions (RES-CLFs), were introduced as methods for achieving (rapidly) exponential stability for walking robots \cite{ames2014rapidly, ames2012control}. 
This control approach has the benefit of yielding an entire class of controllers that provably stabilize periodic orbits for hybrid system models of walking robots, and can be realized in a pointwise optimal fashion. 

In this section, we will introduce an alternative formulation of the CLF with equivalent convergence properties, but with more desirable traits for actual implementation and tuning on hardware. 
These new developments will then form the basis for the experimental study introduced for Cassie in \secref{sec:implementation}.

\subsection{Preliminaries on Control Lyapunov Functions}
Recall the feedback control approach given in \secref{sec:fbl}, which resulted in a linear system that could stabilize the output dynamics. 
As an alternative to the feedback linearizing approach, we can instead propose a control Lyapunov function candidate $V(\eta)$ with $V : Y \rightarrow \R$. A control can then be chosen pointwise in time such that the time derivative of the Lyapunov function $\dot{V}(\eta,\nu) \leq 0$, resulting in stability in the sense of Lyapunov, or $\dot{V}(\eta,\nu) < 0$ for asymptotic stability and $\dot{V}(\eta,\nu) + \gamma V(\eta) \leq 0$ with $\gamma > 0$ for exponential stability. 


\begin{definition}\label{def:res-clf} (RES-CLF \cite{ames2014rapidly})
    For the system \eqref{eq:closed-loop-output-dynamics}, a continuously differentiable function $V_\epsilon : Y \rightarrow \R$ is said to be a \textbf{rapidly exponentially stabilizing control Lyapunov function} if there exist positive constants $c_1,c_2,c_3 > 0$ such that:
    \begin{gather}
        c_1 || \eta(x) ||^2 \leq V_\epsilon(\eta(x)) \leq \frac{c_2}{\epsilon^2} || \eta(x) ||^2, \\
        \inf_{u \in U} \left[ L_{{f}} V_\epsilon(\eta(x),z) + L_{{g}} V_\epsilon(\eta(x),z) u + \frac{c_3}{\epsilon}V_\epsilon(\eta(x)) \right] \leq 0,
    \end{gather}
    for all $0<\epsilon < 1$ and for all $(\eta(x),z) \in Y \times Z$.
\end{definition}

In the context of the control system \eqref{eq:output_dynamics}, we consider the continuous time algebraic Riccati equations (CARE):
\begin{align}
    F^T P + P F - P G R^{-1} G^T P + Q = 0, \label{eq:care}
\end{align}
for $Q=Q^T > 0$, $R = R^T > 0$, and solution $P=P^T > 0$. 
Using \defref{def:res-clf}, we can then construct a (R)ES-CLF: 
\begin{align}
    V(\eta) = \eta^T \underset{P_\epsilon}{\underbrace{\mathbf{I}_\epsilon P \mathbf{I}_\epsilon} \eta}, 
    \hspace{12pt} \mathrm{ with } \hspace{3 pt} \mathbf{I}_\epsilon := \mathrm{diag}\left(  \frac{1}{\epsilon} \mathbf{I}, \mathbf{I} \right), \label{eq:lyap}
\end{align}
where the selection of $0 < \epsilon < 1$ creates a RES-CLF, and $\epsilon = 1$ instead renders an ES-CLF. We can find the derivative of \eqref{eq:lyap} to be:
\begin{align}
    \dot{V}(\eta) &= L_F V(\eta) + L_G V(\eta) \nu, \label{eq:Vdot}
\end{align}
where the Lie derivatives of $V$ along the linear output system's dynamics \eqref{eq:output_dynamics} are:
\begin{align}
    L_F V(\eta) &= \eta^T (F^T P_\epsilon + P_\epsilon F)\eta,\\
    L_G V(\eta) &= 2 \eta^T P_\epsilon G.
\end{align}
An exponential convergence constraint can then be prescribed:
\begin{align}
    L_F V(\eta) + L_G V(\eta) v \leq - \frac{1}{\epsilon}\underbrace{\frac{\lambda_{\mathrm{min}}(Q)}{ \lambda_{\mathrm{max}}(P_\epsilon)}}_{\gamma} V(x), \label{eq:CLFv}
\end{align}
where $\gamma$ is related to the convergence rate. This constraint is in terms of our auxiliary control input $\nu$ and not the actual feedback control $u$. In order to convert back into a form which can be represented in terms of the control input, we can use the previous relationship between $u$ and $\nu$ in \eqref{eq:FBL} 
to obtain the CLF constraint stated in terms of $x$ and $u$:
\begin{align}
    \underbrace{L_F V(x) + L_G V(x) \mathbf{L}_f y(x)}_{L_f V(x)} + \underbrace{L_G V(x) \mathcal{A}(x)}_{L_g V(x)} u \leq - \frac{\gamma}{\epsilon} V(x).
\end{align}
In the context of (R)ES-CLF, we can then define the set
\begin{align}
    K_\epsilon(x) = \{u_\epsilon \in U : L_f V(x) + L_g V(x) u + \frac{\gamma}{\epsilon} V(x) \leq 0 \}, \label{eq:ES_clf_u_class}
\end{align}
consisting of the control values which result in (rapidly) exponential convergence, wherein $\dot{V}(\eta(x)) \leq - \frac{\gamma}{\epsilon} V(\eta(x))$.

\newsec{Stabilizing Hybrid Zero Dynamics.}
Starting from the assumption that a system has stable zero dynamics and is shaped in such a way that it has hybrid invariance, let us consider a hybrid control system \eqref{eq:hybrid_control_system} in normal form
where $\eta,z,\bar{f}$, and $\bar{g}$ are defined as in \eqref{eq:closed-loop-output-dynamics}, $\Delta_{\eta}$ and $\Delta_{z}$ are locally Lipschitz in their arguments, and the domain and guard are now:
\begin{align}
    \Domain &= \{ (\eta,z) \in Y \times \mathcal{Z} ~ | ~ H(\eta,z) \geq 0\}, \\
    \Guard  &= \{ (\eta,z) \in Y \times \mathcal{Z} ~ | ~ H(\eta,z) = 0, \dot{H}(\eta,z) < 0 \} 
\end{align}
where $H(\eta,z):Y\times\mathcal{Z} \rightarrow \R$ is a continuously differentiable function where $L_{\bar{g}} H = 0$. 
If we now assume that the normal-form hybrid control system has continuous invariance, $\bar{f}(0,z) = 0$, and discrete invariance, $\Delta_x(0,z) = 0$, then 
our system has hybrid zero dynamics. In other words, we have encoded satisfaction of the invariance condition, $\Delta ( \mathcal{Z} \cap S ) \subset \mathcal{Z}$, previously given for the full-order system \eqref{eq:resetmapv}. 
If we further assume that a RES-CLF $V_{\epsilon}$ is chosen to obtain a locally Lipschitz control law $u_{\epsilon}(\eta,z) \in K_{\epsilon}(\eta,z)$ that can be applied 
where $u_{\epsilon}(\eta,z) \in K_{\epsilon}(\eta)$ implies $u_{\epsilon}(0,z)=0$  and thus preserves the hybrid zero dynamics $\HybridSystem^\alpha|_\mathcal{Z}$. Because the hybrid zero dynamics holds, the stability of periodic orbits also follow \cite{ames2014human}. In fact, a stronger statement can be made regarding the stability of the hybrid system.
\begin{theorem}\label{thm:res-clf_hzd}(RES-CLF and Hybrid Zero Dynamics \cite{ames2014rapidly}) 
Let $\mathcal{O}|_{\mathcal{Z}}$ be an exponentially stable periodic orbit of the hybrid zero dynamics $\HybridSystem^\alpha|_\mathcal{Z}$ transverse to $S\cap \mathcal{Z}$ and assume there exists a RES-CLF $V_{\epsilon}$ for the continuous dynamics \eqref{eq:eom_nonlinear}.
Then there exists an $\bar{\epsilon}>0$ such that for all $0<\epsilon<\bar{\epsilon}$ and for all Lipschitz continuous $u_{\epsilon}\in K_{\epsilon}(\eta,z)$, $\mathcal{O}=\iota_0(\mathcal{O}|_{\mathcal{Z}})$ is an exponentially stable hybrid periodic orbit of $\HybridSystem^\alpha$.
\end{theorem}
The proof of \thmref{thm:res-clf_hzd} can be found at \cite{ames2014rapidly}, with the primary takeaway being that any RES-CLF controller $u_{\epsilon} \in K_{\epsilon}(\eta,z)$ results in a stable orbit for the full-order dynamics if one exists in the reduced order dynamics.

\newsec{Quadratic Programming and Control Lyapunov Functions.}
The advantage of \eqref{eq:ES_clf_u_class} and \thmref{thm:res-clf_hzd} is that it yields a large set of controllers that can result in stable HZD walking on bipedal robots. 
That is, for any $u \in K_\epsilon(x)$ the hybrid system model of the walking robot, per the HZD framework introduced in \ref{sec:trajectory_hzd}, has a stable periodic gait given a stable periodic orbit in the zero dynamics \cite{ames2014rapidly}. This suggests an optimization-based framework is possible, where the inequality:
\begin{align}
    L_{f} V(x) + L_{g} V(x) u + \frac{\gamma}{\epsilon} V(x) \leq 0, \label{eq:tradCLFconst}
\end{align}
is satisfied in a pointwise-optimal fashion and solved in a QP: 
\begin{align}
    u^{\ast} = \argmin_{u\in U \subset \mathbb{R}^m} \quad & \frac{1}{2}u^T Q(x) u + c^T(x) u \tag{CLF-QP} \label{eq:clf-qp}\\
    \mathrm{s.t.} \quad & L_f V(x) + L_g V(x)u \leq - \frac{\gamma}{\epsilon} V(x) \notag
\end{align}
where $Q\in\R^{m\times m}$ is a symmetric positive-definite matrix and $c\in \R^m$ is vector. 
The choice of $Q(x)$ and $c(x)$ is important in implementation. Specifically, not all choices will result in Lipschitz continuity of the resulting torque, and selecting costs which are inconsistent with the CLF convergence inequality can cause the controller to chatter. One common choice is to use the fact that the preliminary feedback control law in the HZD and CLF constructions is feedback linearization: 
\begin{align*}
    u^{\ast}(x) = \argmin_{u\in U \subset \mathbb{R}^m} \quad & ||\mathcal{A}(x)u + \mathbf{L}_f y(x)||^2\\
    \mathrm{s.t.} \quad & L_f V(x) + L_g V(x)u \leq - \frac{\gamma}{\epsilon} V(x)
\end{align*}
with $Q(x) = \mathcal{A}^T(x) \mathcal{A}(x)$ and $c^T(x) = 2 (\mathbf{L}_f y(x))^T \mathcal{A}(x)$ in terms of the original cost. 

For the holonomic constraints to be satisfied in the dynamics \eqref{eq:eom_nonlinear}, and thus in the QP constraint \eqref{eq:tradCLFconst}, we must either augment $u$ with $\lambda$ as an additional decision variable \cite{hereid2014embedding, ames2013towards}, or solve for the generalized force explicitly:
 \begin{align}
    \lambda(x,u) &= (J_c D^{-1} J^T_c)^{-1} \left( J_c D^{-1} (H - B u) - \dot{J}_c \dot{q} \right), \label{eq:analytic_constraint_force}
\end{align}
and substitute back into the expression \eqref{eq:eom_nonlinear} to remove it. 

\newsec{Constraint Relaxation.}
The optimization formulation of CLFs allows for additional constraints and objectives to be incorporated into the optimization. These constraints can include various things which are important for realization on actual robotic platforms such as torque constraints for input saturation, friction constraints, or unilaterality conditions on contact forces.
However, one of the downsides to incorporating additional constraints into the problem is that it may not be possible to satisfy them concurrently with \eqref{eq:tradCLFconst}. Meaning that a relaxation, $\delta$, must be added to penalize violation of \eqref{eq:tradCLFconst}:
\begin{align}
    u^{\ast} = \argmin_{u,\delta} \quad & \frac{1}{2} u^T Q(x) u + c^T(x) u + \rho \delta^2 \tag{CLF-QP-$\delta$}\label{eq:clf-qp-delta}\\
    \mathrm{s.t.} \quad & L_f V(x) + L_g V(x)u \leq - \frac{\gamma}{\epsilon} V(x) + \delta \notag\\
                        & C_I(x) u \leq d_i(x) \notag
\end{align}
where $\rho$ is a large weight penalizing the relaxation.

\subsection{Inverse Dynamics-Based Control Lyapunov Functions}
\newsec{Inverse Dynamics.}
Inverse dynamics is a widely used method to approaching controller design for achieving a variety of motions and force interactions, typically in the form of task-space objectives. Given a target behavior, the dynamics of the robotic system are inverted to obtain the desired torques. 
In many recent works, variations of these approaches have been shown to allow for high-level tasks to be encoded with intuitive constraints and costs in optimization based controllers, some examples being \cite{apgar2018fast, kuindersma2016optimization, koolen2016design, herzog2016momentum, kuindersma2014efficiently}. 

Here, we present a minimal implementation of an inverse dynamics controller. 
The inverse dynamics problem can also be posed using a QP to exploit the fact that the instantaneous dynamics and contact constraints can be expressed linearly with respect to a certain choice of decision variables. Specifically, let us consider the set of optimization variables $\mathcal{X} = [ \ddot{q}^T, u^T, \lambda^T]^T \in \Xext := \mathbb{R}^n \times U \times \mathbb{R}^{m_h}$, which are linear with respect to \eqref{eq:eom} and \eqref{eq:hol_accel}:
\begin{align}
    \begin{bmatrix} D(q) & -B & - J_h(q)^T  \\ J_h(q) & 0 & 0 \end{bmatrix} \mathcal{X} + \begin{bmatrix} H(q,\dot{q}) \\ \dot{J}_h(q)\dot{q} \end{bmatrix} = 0.  \label{eq:IDform_intro}
\end{align}
Also consider a Cartesian objective in the task space of the robot, which can be characterized using: \eqref{eq:ddqexplicit}:
\begin{align}
    J_y(q,\dot{q}) \ddot{q} + \dot{J}_y(q,\dot{q}) \dot{q} - \ddot{y}_2^* = 0:
\end{align}
where $\ddot{y}_2^* = K_P y_2 + K_D \dot{y}_2$ is a PD control law which can be tuned to achieve convergence. 
In it's most basic case, we can combine these elements to pose this QP tracking problem as:
\begin{align*}
    \mathcal{X}^{\ast}(x) = \argmin_{\mathcal{X}\in \Xext} \quad & || J_y(q) \ddot{q} + \dot{J}_y(q,\dot{q}) \dot{q} - \ddot{y}^* ||^2 \tag{\text{ID-QP}}\label{eq:id-qp}\\
   \textrm{ s.t.} \quad & \text{\eqnref{eq:IDform_intro}} \tag{\text{System Dynamics}} \\
   & u_{\mathrm{min}} \leq u \leq u_{\mathrm{max}} \tag{\text{Torque Limits}} \\
    & \text{\eqnref{eq:pyramid_friction}} \tag{\text{Friction Pyramid}}
\end{align*}
where we have included feasibility constraints such as the friction cone (\eqnref{eq:cone_friction}) and torque limits. 
Although this controller satisfies the contact constraints and yields an approximately optimal solution to tracking task-based objectives, it does not provide formal guarantees with respect to stability. In dynamic walking motions this becomes an important consideration,  
wherein impacts and footstrike can destabilize the system, requiring more advanced nonlinear controllers.

\newsec{Revisiting Feedback Linearization.}
Taking inspiration from inverse dynamics approaches, we return to \secref{sec:fbl}, where the auxiliary control input, $\nu$, for a feedback linearizing controller is set to equal the second time derivative of the outputs. 
Recall that we had taken derivatives of the outputs along $f(x)$ and $g(x)$ we obtain $\ddot{y}_2 = \mathbf{L}_f y(x) + \mathcal{A}(x) u$. 
This can equivalently be done by taking the derivatives of the outputs \eqref{eq:outputs} in terms of acceleration instead of along \eqref{eq:eom_nonlinear}:
\begin{align}
    \ddot{y}_2  &= \underbrace{ \frac{\partial }{\partial q} \left(\frac{\partial y_2}{\partial q}\dot{q} \right) }_{\dot{J}_y} \dot{q} + \underbrace{\frac{\partial y_2}{\partial q}}_{J_y} \ddot{q}. \label{eq:ddqexplicit}
\end{align}
We then return to the definition of $\eta$ where $\nu = \ddot{y}_2$. Rather than directly choosing an input, $u$, we can instead solve for an acceleration, $\ddot{q}$, that generates an equivalent response. 

\begin{theorem} \label{thm:accel_equiv}
    For a robotic system with dynamics \eqref{eq:eom} and outputs of the form \eqref{eq:outputs}, where $D(q)$ is positive definite (and therefore invertible), and independent outputs are chosen (i.e., 
    $J_y(q)$ is not rank-deficient), then any controller in the set:
    \begin{align}
        K_{IO}(q,\dot{q}) = \{u\in U : \ddot{q} =J^{\dagger}_y(q)(-\dot{J}_y(q,\dot{q})\dot{q} + \nu)  \},
    \end{align}
    elicits the same response in the output dynamics as the feedback linearizing input $u = \mathcal{A}^{-1}(x)(-\mathbf{L}_f y(x)+ \nu)$.
\end{theorem}
\begin{proof}
    Using \eqref{eq:ddqexplicit}, $\ddot{q}$ can be chosen to satisfy
    \begin{align}
        \begin{bmatrix} \dot{y}_1 \\ \ddot{y}_2 \end{bmatrix} &= J_y(q)\ddot{q} + \dot{J}_y(q,\dot{q})\dot{q} = \nu.
    \end{align}
    By constraining $\ddot{q} = J^{\dagger}_y(q)(-\dot{J}_y(q,\dot{q})\dot{q} + \nu)$, 
    where $J^{\dagger}_y$ is a right pseudo inverse of the full rank matrix $J_y$, with $J_y J^{\dagger}_y = I$, and the outputs evolve as:
    \begin{align*}
        J_y(q)\ddot{q} + \dot{J}_y(q,\dot{q})\dot{q} &= \dot{J}_y\dot{q} + J_yJ^{\dagger}_y(-\dot{J}_y\dot{q} + \nu) \\
                &= \nu 
                = \begin{bmatrix} \dot{y}_1 \\ \ddot{y}_2 \end{bmatrix} 
                = \mathbf{L}_f y(x) + \mathcal{A}(x) u.
    \end{align*}
\end{proof}

\newsec{Inverse Dynamics Quadratic Programs with CLFs.}
In this section, we return to the concept of a QP which can solve the inverse dynamics problem for a floating-base robot. 
Despite the connections shown between inverse dynamics and feedback linearization, CLF based controllers have only been successfully implemented on hardware 
on low-dimensional robots with 
an analytical solution \cite{ames2014rapidly}, as a 
minimal QP \cite{galloway2015torque}, or indirectly by simulating the nominal system and tracking the resulting motion  
\cite{cousineau2015realizing}.
There are several issues we suggest may be influencing this lack of successful implementations.

The first significant difficulty in realizing optimization and model-based controllers, and therefore in implementing CLFs, is in obtaining accurate models for these complex robotic platforms. 
In the earliest implementations of CLFs, significant system identification was necessary 
\cite{park2011identification}. 
To mitigate these issues, 
robust CLF formulations have been proposed \cite{nguyen2016optimal}, or machine learning methods to account for unmodeled dynamics \cite{taylor2019episodic} \cite{choi2020reinforcement}.
While these 
discrepancies may be large in some cases, the successes of \eqref{eq:id-qp} controllers on 
complex humanoids as ATLAS \cite{kuindersma2016optimization} shows that it is possible. Thus the aim of the approach presented in this paper will not directly address model uncertainty and will instead focus on how the formulation of the problem can influence its behavior.

\begin{definition} \label{def:id-clf}
    Given a set of outputs \eqref{eq:outputs} for the 
    robotic 
    system of 
    \eqref{eq:eom} and \eqref{eq:hol_accel}, the \textbf{inverse dynamics control Lyapunov function quadratic program} (ID-CLF-QP) with decision variables $\mathcal{X} = [ \ddot{q}^T, u^T, \lambda^T]^T \in \Xext := \mathbb{R}^n \times U \times \mathbb{R}^{m_c}$ is: 
    \begin{align}
        \mathcal{X}^{\ast} = \argmin_{\mathcal{X}\in \Xext} \quad & \frac{1}{2}\mathcal{X}^T Q \mathcal{X} + c^T \mathcal{X} \tag{ID-CLF-QP} \label{eq:id-clf-qp}\\
        \textrm{s.t.} \quad &  L_F V + L_G V \left(\dot{J}_y\dot{q} + J_y\ddot{q} \right) \leq - \frac{\gamma}{\epsilon} V \notag\\
        &D\ddot{q} + H = Bu + J_c^T\lambda \notag\\
        &J_c\ddot{q} + \dot{J}_c\dot{q} = 0 \notag
    \end{align} 
    with $Q(x) = Q^T(x) > 0$ and real vector $c(x)\in \R^{n+m+m_c}$. The solution $\mathcal{X}^*$ gives an associated controller $k_{\mathrm{idclf}}^*(x)$.
\end{definition}
Where we have termed the QP with the phrase ``inverse dynamics'' as it is determining a control input, $u$, based on convergence criteria imposed on the generalized accelerations, $\ddot{q}$, through satisfaction of the equations of motion \eqref{eq:eom} and \eqref{eq:hol_accel}. 
Perhaps the most significant observation of \defref{def:id-clf} is that we have traded an increased number of decision variables for a set of equality constraints that do not require any matrix inversions.
This is 
relevant to implementation as it has been shown that the condition number of the joint space inertia matrix increases quartically with the length of a kinematic chain \cite{featherstone2014rigid}. Repeated inversions of this matrix 
thus may be an obvious source of numerical stiffness 
and therefore controller degradation \cite{nakanishi2008operational}. 
In addition, performing the required inversions for evaluating 
\eqref{eq:eom_nonlinear} are computationally expensive, and can violate strict timing requirements 
on hardware. 

This section will construct a framework around the \eqref{eq:id-clf-qp} controller, 
motivating its use as a stabilizing controller for HZD locomotion. 
Specifically, let $\mathcal{O}$ be a periodic orbit of the zero dynamics $\dot{z} = \omega(0,z)$ and assume that $\mathcal{O} \subset \mathcal{Z}$ is exponentially stable. Then the following result states that the resulting controller from \eqref{eq:id-clf-qp} can stabilize $\mathcal{O}$ in the full-order dynamics.
\begin{theorem} \label{thm:main_idclf_thm} (ID-CLF-QP and HZD)
    Assume that \eqref{eq:id-clf-qp} is locally Lipschitz and unique for all points in a neighborhood of an exponentially stable periodic orbit, $\mathcal{O}|_\mathcal{Z}$, of the hybrid zero dynamics $\HybridSystem^\alpha|_{\mathcal{Z}}$ transverse to $S\cap \mathcal{Z}$. Then for the \eqref{eq:id-clf-qp} controller with choice of RES-CLF, $V_\epsilon(x)$, $u_\epsilon^*(x) = k_{\mathrm{idclf},\epsilon}(x)$ there exists an $\bar{\epsilon}>0$ such that for all $0<\epsilon<\bar{\epsilon}$, $\mathcal{O}=\iota_0(\mathcal{O}|_{\mathcal{Z}})$ is an exponentially stable hybrid periodic orbit of $\HybridSystem_\epsilon$.
\end{theorem}

Before we proceed with a proof for \thmref{thm:main_idclf_thm}, we will first establish that the pointwise optimal control action obtained from \eqref{eq:id-clf-qp} in fact renders stability of the transverse dynamics \eqref{eq:output_dynamics} in a similar manner to \eqref{eq:clf-qp}. Using this result, we will then show that the \eqref{eq:id-clf-qp} can be transformed into an equivalent \eqref{eq:clf-qp}.

\begin{lemma}\label{lem:id-clf-transverse-stability}
    The pointwise optimal solution $\mathcal{X}^* = k_{\mathrm{idclf}}^*(x)$ of \eqref{eq:id-clf-qp} yields a control action within the set of admissible inputs for a CLF given by \defref{def:res-clf}:
    \begin{align*}
        k_{\mathrm{idclf}}^*(x) \in K_u(x) = \left\{ u \in \R^m : L_f V + L_g V u \leq - \frac{\gamma}{\epsilon} V \right\}. 
    \end{align*}
    As a result, if $u^*$ taken from $\mathcal{X}^*$ is locally Lipschitz and if the zero dynamics $\dot{z}=\omega(0,z)$ is locally exponentially stable, then \eqref{eq:id-clf-qp} is a locally exponentially stabilizing controller for the closed-loop system in \eqref{eq:closed-loop-output-dynamics}.
\end{lemma}
\begin{proof}
Application of \thmref{thm:accel_equiv} means that the collection of constraints:
\begin{align*}
    \begin{cases}
        &L_F V(x) + L_G V(x) \left(\dot{J}_y(q,\dot{q})\dot{q} + J_y(q)\ddot{q} \right) \leq - \frac{\gamma}{\epsilon} V(x) \notag\\
        &D(q)\ddot{q} + H(q,\dot{q}) = Bu + J_c^T(q)\lambda \notag\\
        &J(q)\ddot{q} + \dot{J}_c(q)\dot{q} = 0 \notag
    \end{cases}
\end{align*}
can be rewritten as a single inequality by solving \eqref{eq:eom} and \eqref{eq:hol_accel} for $\ddot{q}$ and substituting into the CLF inequality to obtain:
\begin{align}
    L_f V(x) + L_g V(x) u \leq - \frac{\gamma}{\epsilon} V(x),
\end{align}
which is the convergence condition required for exponential convergence provided in \defref{def:res-clf}. Because we can analytically show this equivalence, the existing CLF convergence conditions in \defref{def:res-clf} apply to \eqref{eq:id-clf-qp}.
For any Lipschitz continuous feedback control law $u \in K_u(x)$, the inequalities in \defref{def:res-clf} imply that the solutions to \eqref{eq:closed-loop-output-dynamics} 
satisfy (with $\gamma = c_3$): 
\begin{align*}
\dot{V}(\eta,u^*(\eta,z)) \leq - \frac{c_3}{\epsilon} V(\eta)
\ & \Rightarrow \ V(\eta(t)) \leq e^{-\frac{c_3}{\epsilon} t} V(\eta(0)) \\
\ & \Rightarrow \ 
\| \eta(t) \| \leq \sqrt{ \frac{c_2}{c_1}} e^{- \frac{c_3}{2\epsilon} t } \| \eta(0) \|. 
\nonumber
\end{align*}
\end{proof}

\begin{remark}
    As was stated earlier in this section for \eqref{eq:clf-qp}, not all choices of the cost terms for $Q(x)$ and $c(x)$ for \eqref{eq:id-clf-qp} will result in Lipschitz continuity of the resulting QP controller \cite{morris2015continuity}. If they are selected in a way which conflicts with the convergence constraint then the input can instantaneously change and create a discontinuity.
\end{remark}

One of the most important consequences of \lemref{lem:id-clf-transverse-stability} is that we can pose a \eqref{eq:id-clf-qp} controller to stabilize the zero dynamics surface during continuous phases of motion for underactuated robotic systems. In fact, we can also show that a \eqref{eq:id-clf-qp} can be analytically converted into an \eqref{eq:clf-qp}.
\begin{lemma} \label{lem:id-clf-equiv}
    For any given cost $\mathcal{J}_{\mathrm{id}}(x,\mathcal{X}) = \frac{1}{2}\mathcal{X}^T Q_{\mathrm{id}}(x) \mathcal{X} + c_{\mathrm{id}}^T(x) \mathcal{X}$ with $Q_{\mathrm{id}}(x) = Q_{\mathrm{id}}^T(x)>0$ and real vector $c_{\mathrm{id}}(x)\in \R^{n+m+m_c}$ of \eqref{eq:id-clf-qp}, there exists a cost $\mathcal{J}_{\mathrm{u}}(x,u)$ which is quadratic with respect to $u$ for \eqref{eq:clf-qp} such that the problems are analytically equivalent.
\end{lemma}
\begin{proof}
To begin, we will simply establish a linear transformation between the decision variables and then plug them into $\mathcal{J}_{\mathrm{id}}(x,\mathcal{X})$ to find a cost. 
We can substitute the analytic expression for the constraint force  \eqref{eq:analytic_constraint_force} into \eqref{eq:eom} 
to obtain the constrained equations of motion:
\begin{align*}
    D(q) \ddot{q} + \hat{H}(q,\dot{q}) = \hat{B}(q,\dot{q}) u,
\end{align*}
where we can solve for $\ddot{q}$, and thus form an expression relating $u$ to $\mathcal{X}$:
\begin{align*}
    \mathcal{X} &=  \underbrace{\begin{bmatrix} 
        D^{-1} \hat{B}  \\ 
        I \\ 
        \mathbf{A}^{-1} J_c D^{-1} \hat{B}
    \end{bmatrix}}_{A_u} u +
    \underbrace{\begin{bmatrix}
        -D^{-1} \hat{H} \\
        0 \\
        \mathbf{A}^{-1} \left(\dot{J}_c\dot{q} - J_c D^{-1} H \right)
    \end{bmatrix}}_{b_u},
\end{align*}
where $\mathbf{A}:=\left[ J_c D^{-1} J_c^T\right]$. Directly substituting this relation yields a quadratic cost:
\begin{align}
    \mathcal{J}_{\mathrm{id}}(x,\mathcal{X}) &= \mathcal{X}^T Q_{\mathrm{id}} \mathcal{X} + 2 c_{\mathrm{id}}^T \mathcal{X} \notag\\
    &= \left[ A_u u + b_u \right]^T Q_{\mathrm{id}} \left[ A_u u + b_u \right] + 2 c_{\mathrm{id}}^T \left[A_u u + b_u \right] \notag\\
    &= u^T A_u^T Q_{\mathrm{id}} A_u u + 2 \left[ b_u^T Q_{\mathrm{id}} + c_{\mathrm{id}}^T \right]A_u u \notag \\
    &\hspace{40mm} + \left[ b_u^T Q_{\mathrm{id}} b_u + 2 c_{\mathrm{id}}^T b_u \right] \notag\\
    &=: \mathcal{J}_{u}(x,u) \label{eq:id-clf-cost-relation}
\end{align} 
We have already shown in the proof for \lemref{lem:id-clf-transverse-stability} that the constraints for \eqref{eq:id-clf-qp} reduce to \eqref{eq:clf-qp}.
\end{proof}

\begin{proof}[\unskip\nopunct]\newsec{Proof of \thmref{thm:main_idclf_thm}:}

    Let us begin by posing a \eqref{eq:id-clf-qp} with a RES-CLF condition \eqref{eq:ES_clf_u_class} on the CLF convergence:
    \begin{align*}
        \mathcal{X}^{\ast}_\epsilon = \argmin_{\mathcal{X}\in \Xext} \quad & \frac{1}{2}\mathcal{X}^T Q_{\mathrm{id}}(x) \mathcal{X} + c^T_{\mathrm{id}}(x) \mathcal{X}  \\
        \textrm{s.t.} \quad &   L_F V_\epsilon + L_G V_\epsilon \left(\dot{J}_y\dot{q} + J_y\ddot{q} \right) \leq - \frac{\gamma}{\epsilon} V_\epsilon \\
        &D\ddot{q} + H = Bu + J_c^T\lambda \\
        &J_c\ddot{q} + \dot{J}_c\dot{q} = 0 
    \end{align*} 
    The primary consequence of \lemref{lem:id-clf-transverse-stability} and \lemref{lem:id-clf-equiv} is that this problem will render an analytically equivalent control action to the RES-\eqref{eq:clf-qp}:
    \begin{align*}
        u^{*}_\epsilon = \argmin_{u\in U \subset \mathbb{R}^m} \quad & \mathcal{J}_{u}(x,u) \\
        \mathrm{s.t.} \quad & L_f V(x) + L_g V_\epsilon(x)u \leq - \frac{\gamma}{\epsilon} V_{\epsilon}(x) 
    \end{align*}
    Thus, the control action belongs to the family of RES-CLF controllers given by $k_{\mathrm{idclf},\epsilon}^*(x) \in K_{\epsilon}(x)$ where:
    \begin{align*}
        K_{\epsilon}(x) = \left\{ u \in \R^m : L_f V(x) + L_g V_\epsilon(x)u \leq - \frac{\gamma}{\epsilon} V_{\epsilon}(x) \right\}.
    \end{align*}
    Because the pointwise optimal control action is thus a RES-CLF, if $\mathcal{X}^{\ast}_\epsilon$ is locally Lipschitz and unique then \thmref{thm:res-clf_hzd} applies to $k_{\mathrm{idclf},\epsilon}^*(x)$, completing the proof.
\end{proof}

\subsection{Cost Function Design}
One of the clear benefits to \eqref{eq:id-clf-qp} is that there exists a wide range of costs that can be designed without needing complex expressions, as the decision variables are affine with respect to mostly kinematic matrices in the equations of motion \eqref{eq:eom}-\eqref{eq:hol_accel} and output dynamics \eqref{eq:output_dynamics}. 
First, let us consider the feedback linearizing acceleration response:
\begin{align}
    J_{\mathrm{IO}}(x,\mathcal{X}) := ||J_y(q) \ddot{q}_y + \dot{J}_y(q,\dot{q}) \dot{q}||^2. \label{eq:fbl_acc_cost}
\end{align}
The equivalence of the expression in $J_{\mathrm{IO}}(x,\mathcal{X})$ to \eqref{eq:FBL} can be seen by \thmref{thm:accel_equiv}.
Perhaps the most important observation that we should make is that in order for \eqref{eq:id-clf-qp} to be solved uniquely, the Hessian matrix, $Q(x)$, must be positive definite and therefore also full rank. 
One of the most common ways to address this is to regularize the decision variables:
\begin{align}
    J_{\mathrm{reg}}(x,\mathcal{X}) = || \mathcal{X} - \mathcal{X}^*_\alpha ||^2. \label{eq:reg_cost}
\end{align}

In fact, if we have solved for a stable walking gait using the HZD methodology using \eqref{eq:opteqs}, then we have already obtained a parameterized piecewise polynomial for $\mathcal{X}^*_\alpha(t)$ when the robot is on the stable hybrid periodic orbit.
\begin{proposition}\label{prop:collocation_gives_parameters} (See \cite{hereid2016dynamic}, Chapter 4)
    Suppose that $\mathbf{w(\alpha^*)}$ describes a feasible walking gait solving \eqref{eq:opteqs}. Then the piecewise polynomial solution $\phi^*(t)$ determined by the NLP solution,  $\{T^*,q^*(t),\dot{q}^*(t),\ddot{q}^*(t),\lambda^*(t),u^*(t) \}$, is hybrid invariant under the virtual constraint feedback control law \eqref{eq:FBL} with parameters $\alpha^*$, i.e. $\phi^*(t) \subset \mathcal{Z}_{\alpha^*}$.
\end{proposition}

\begin{proposition} \label{prop:regularization_id-clf}
    Consider an \eqref{eq:id-clf-qp} with the cost:
    \begin{align}
        \mathcal{J}(x,\mathcal{X}) = \mathcal{J}_{z}(x,\mathcal{X}) + || \mathcal{X} - \mathcal{X}^*_\alpha(t) ||^2  \notag
    \end{align}
    where $\mathcal{J}_{z}(x,\mathcal{X})$ is defined in such a way that $\mathcal{J}_{z}(\mathcal{X})|_{\mathcal{Z}} \equiv 0$ (we can see that \eqref{eq:fbl_acc_cost} is an example of such a cost as $y(x) \equiv 0$ when $\eta(x) \equiv 0$). Then when the robot is on the zero dynamics (i.e. $\dot{z} = \omega(0,z)$), the optimal control action is $\mathcal{X}^* = \mathcal{X}^*_\alpha(t)$. 
\end{proposition}
\begin{proof}
    \propref{prop:collocation_gives_parameters} means that the solution to \eqref{eq:opteqs} lies on the hybrid invariant zero dynamics surface of the corresponding walking gait. Thus, when the robot is on the zero dynamics surface the QP constraints vanish, since they are implicitly satisfied on the solution $\phi^*(t) \subset \mathcal{Z}_{\alpha^*}$ if $\mathbf{w(\alpha^*)}$ is a feasible solution to \eqref{eq:opteqs}. Further, the cost $\mathcal{J}_{z}(\mathcal{X})|_{\mathcal{Z}} \equiv 0$ by definition, and thus the 
optimal control action is given by $\mathcal{X}^* = \mathcal{X}^*_\alpha(t)$. 
\end{proof}

\begin{figure*}[t!]
	\centering
	\includegraphics[width=0.92\textwidth]{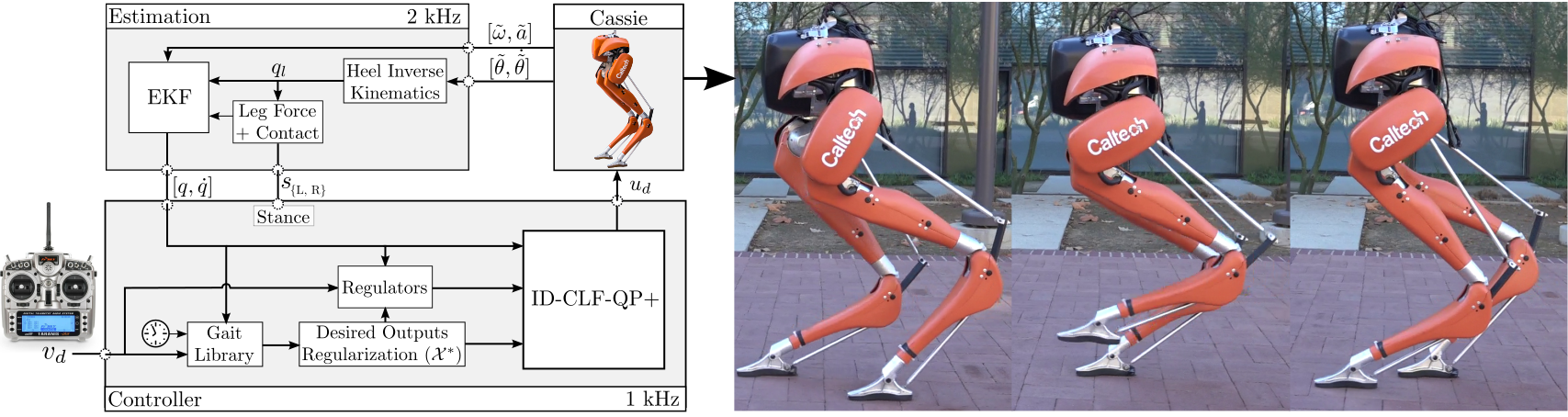}
	\caption{The control and estimation diagram for implementation on Cassie. The estimation and controller run in separate threads on the Intel NUC at $2$ kHz and $1$ kHz, respectively. Also shown are gait tiles for Cassie walking on a sidewalk at Caltech while using a \eqref{eq:id-clf-qp-plus} controller.}
	\label{fig:cassie_clf_forward_tiles}
\end{figure*}

\newsec{Constraint Relaxation.}
Up until this point in the development of \eqref{eq:id-clf-qp} controllers, we have considered only the dynamics and CLF constraints applied to the QP. However, in order to implement these controllers on hardware, we 
require additional constraints for 
For walking robots, these are typically 
torque limits and admissibility conditions on the constraint forces (see \eqref{eq:pyramid_friction} and \eqref{eq:footroll}). 
Due to the presence of these constraints, it is not always feasible for the system to simultaneously satisfy physical constraints and converge according to the CLF bound \cite{galloway2015torque}. The accepted way of dealing with this within the literature \cite{galloway2015torque} is to add a relaxation term, $\delta$, to the convergence constraint with an associated weight, $\rho$. 
Following the development of the relaxed \eqref{eq:clf-qp-delta} constraint, we propose a similar relaxed CLF constraint: 
\begin{align}
    L_F V + L_G V \left(\dot{J}_y\dot{q} + J_y\ddot{q} \right) \leq - \gamma V + \delta.
\end{align}
Because we have introduced a weighted relaxation to the inequality that is minimized in the cost, we can actually 
move the constraint 
into the cost as an exact penalty function \cite{han1979exact}:
\begin{align}
    \mathcal{J}_{\mathrm{\delta}} = \frac{1}{2} \mathcal{X}^T Q(x) \mathcal{X} + c^T(x) \mathcal{X} + \rho || g^+(q,\dot{q},\ddot{q}) ||
\end{align}
where:
\begin{align*}
        g(q,\dot{q},\ddot{q}) &:= L_F V + L_G V \left(\dot{J}_y\dot{q} + J_y\ddot{q} \right) + \gamma V \\
        g^+(q,\dot{q},\ddot{q}) &\triangleq \max(g,0) .
\end{align*}
One of the downsides to using this approach is that the cost term $|| g^+(q,\dot{q},\ddot{q}) ||$ is non-smooth. 
Instead, we can allow $g(q,\dot{q},\ddot{q})$ to go negative, meaning that the controller will always drive convergence even when the inequality \eqref{eq:tradCLFconst} is not triggered. This will lead to a smooth term in the cost, which must be balanced with it's possibly more aggressive control action. We can then remove the CLF convergence inequality from \eqref{eq:id-clf-qp} to obtain a relaxed controller:
\begin{align}
        \mathcal{X}^{\ast}= \argmin_{\mathcal{X}\in \Xext} &\quad  \frac{1}{2}\mathcal{X}^T Q \mathcal{X} + c^T \mathcal{X} + \dot{V}(x,\mathcal{X}) \tag{ID-CLF-QP$^{+}$} \label{eq:id-clf-qp-plus}\\ 
    \textrm{s.t.} \quad &  D(q)\ddot{q} + H(q,\dot{q}) = Bu + J^T(q)\lambda \notag\\
    &  J(q) \ddot{q} + \dot{J}(q,\dot{q})\dot{q} = 0 \notag \\
    & A_E(x) \mathcal{X} = b_e(x)    \notag \\
	& C_I(x) \mathcal{X} \geq d_I(x) \notag
\end{align}
which incentivizes convergence. Further, \textit{whenever it is feasible to do so this problem will render $\dot{V}$ as negative as possible}.

\section{Implementation} \label{sec:implementation}
This section discusses the main experimental result of this paper, and serves to illustrate how the various concepts introduced throughout this work can be combined on hardware. 
Specifically, the motion library developed 
in \secref{sec:trajectory_hzd} on the compliant model of Cassie in \secref{sec:robotmodel} is parameterized and then combined with the \eqref{eq:id-clf-qp} controller of \secref{sec:id-clf}.

There is often an ``artful implementation'' step that translates model-based controllers to a form that can actually implemented on hardware.  Ideally, methods can be developed that allow the exact transcription of model-based methods to hardware in a robust fashion and without heuristics. 
This work serves as a major step in this direction, with the walking shown in this section being the \textit{first successful experimental implementation of a (relaxed) CLF for walking on a 3D biped}.

\subsection{Feedback Controller Development} \label{sec:cassie_clf_controller}
In this section, the feedback controller used on hardware to track the HZD locomotion problem is developed. 
Several notable modifications were made in this section in order to achieve successful walking. 
First, several of the holonomic constraints are removed, 
leaving the ground reaction forces as a decision variable.
Next, the stance 
springs are enforced via a 
soft holonomic constraint.
By adding the spring forces in this way, we can allow the QP to choose the spring torque. 
The springs are then regularized against the measured spring force, with the 
soft constraint 
allowing for non-zero accelerations.

\newsec{Gait Regularization}
The trajectory optimization which was performed in \secref{sec:trajectory_hzd} to obtain a motion library of $171$ individual walking gaits on Cassie was directly used in this section. 
Thus, in order to form a full parameterization:
\begin{align*}
    \mathcal{X}^*_\alpha(t) := \begin{bmatrix} \ddot{q}^*_\alpha(t,\bar{v}^a_{k-1})^T, & 
    u^*_\alpha(t,\bar{v}^a_{k-1})^T, &
    \lambda^*_{c,\alpha}(t,\bar{v}^a_{k-1})^T \end{bmatrix}^T,
\end{align*}
we extract B\'ezier polynomial
coefficients via \eqref{eq:cassie_library_interp} at the average velocity of the previous step, $\bar{v}^a_{k-1}$. 

\newsec{RES-CLF Specification.}
In order to track the virtual constraints in each of the continuous domains, a RES-CLF \eqref{eq:ES_clf_u_class} is found using \eqref{eq:care}. The matrices $Q$ and $R$ are chosen as diagonal matrices, with specific gains given in \tabref{tab:clf_regularization_param}. The solution to \eqref{eq:care}, $P$, is thus a symmetric block matrix:
\begin{align*}
    P = \begin{bmatrix} 
            P_{\mathrm{ud}} & P_{\mathrm{od}} \\
            P_{\mathrm{od}} & P_{\mathrm{ld}}
    \end{bmatrix},
\end{align*}
where $P_{\mathrm{ud}}$, $P_{\mathrm{ld}}$ and $P_{\mathrm{od}}$ are found to be the diagonal matrices:
\begin{align*}
    P_{\mathrm{ud}} &= \mathrm{diag}\left(  795, 683, 137, 880, 15533, 796, 442, 189, 303    \right), \\
    P_{\mathrm{ld}} &= \mathrm{diag}\left(  10.5, 11.3, 8.2, 10.9, 15, 12.5, 10.2, 9.8, 9.4 \right), \\
    P_{\mathrm{od}} &= \mathrm{diag}\left(  60.7, 60.3, 23.4, 56.6, 82.9, 63.7, 43.3, 27, 32.3    \right).
\end{align*}
Using this solution, we select $\epsilon=0.1$ and then construct a RES-CLF according to \eqref{eq:lyap}.

\newsec{Robot Dynamics and Partial Constraint Elimination.}
The holonomic constraint vector given in \secref{sec:robotmodel} contains several constraints which are a function of the internal kinematics of the leg, and do not lend any use in shaping as a decision variable. 
In addition, the measured spring forces on the stance leg are a potential source of model uncertainty, vibration, and numerical stiffness.
To address this, we will differentiate between \textit{hard} and \textit{soft} constraints. Hard constraints, such as the holonomic equality constraint equation \eqref{eq:hol_accel}, 
cannot be violated. Soft constraints, however, are penalized in the cost: 
\begin{align}
    J_s\ddot{q}+ \dot{J}_s\dot{q} = 0 \ \ \Rightarrow \ \ \left|\left|{\begin{bmatrix} J_s & 0 & 0\end{bmatrix}}\mathcal{X} + \dot{J}_s\dot{q} \right|\right|^2 . \label{eq:hol_soft_constraint}
\end{align}
The formulation of holonomic constraints in this way allows for small violations, which is sometimes necessary in practice where systems can be significantly perturbed. 

We therefore partition the constraints, and append the stance spring forces to the soft constraint wrench vector:
\begin{align*}
    \lambda_c := \begin{bmatrix} \lambda_{4\mathrm{bar}} \\ 
                                        \lambda_\text{ns,sp} \end{bmatrix} \in \R^{4},  \quad 
    \lambda_s := \begin{bmatrix} \lambda_\text{sf} \\ \lambda_\text{s,sp} \end{bmatrix} \in \R^{7},
\end{align*}
where $\lambda_c$ denotes constraints which are ``hard'' and $\lambda_s$ denotes the ``soft'' constraints. 
The soft constraint can be weighted to allow violations, and the regularization, $\mathcal{X}^*_\alpha(t)$, is augmented to include the measured spring forces:
\begin{align}
    \lambda^*_{\text{s,sp}} := - K_{\text{sp}} q_{\text{sp}} = - \begin{bmatrix} 2300~ q_{\text{s,sp}} \\ 2000~ q_{\text{s,hs}} \end{bmatrix} .
\end{align}

Instead of computing the hard constraints as variables in the QP, we remove them from the dynamics expressions by forming a linear projection operator \cite{mistry2010inverse,aghili2005unified}, $P_c(q) = I - J_c^\dagger(q) J_c(q)$, 
where $(\cdot)^\dagger$ denotes the pseudoinverse. 
Using this, we can obtain the constrained dynamics \cite{mistry2010inverse}:
\begin{align*}
    D_c(q) \ddot{q} + H_c(q,\dot{q}) \dot{q} = B_c(q) u + J_{s}^T(q) \lambda_{s},
\end{align*}
where the individual terms are:
\begin{align*}
    D_c(q) &= D(q) + P_c(q) D(q) - (P_c(q) D(q))^T, \\
    H(q,\dot{q}) &= P_c(q) H(q,\dot{q}) + D(q) J_c^\dagger(q) \dot{J}_c(q,\dot{q}), \\
    B_c(q) &= P_c(q) B, \\
    J_{c,s}^T(q) &= P_c(q) J_s^T(q).
\end{align*}
This does not significantly complicate the equations of motion, as $P_c(q)$ is simply a function of the internal leg kinematics forming the multi-bar mechanisms on each leg. 

Because the soft constraints are no longer an equality constraint, any cost terms which involve feedback on the generalized accelerations could incentivise their violation. We can instead modify these constraints so that the acceleration component of the cost implicitly satisfies \eqref{eq:orthog_accels}: 
\begin{align}
    \ddot{q}^\perp_y = \left( I - J_s^\dagger(q) J_s(q) \right)\ddot{q} + J_s^\dagger(q) \dot{J}_s(q,\dot{q}) \dot{q}. \label{eq:orthog_accels}
\end{align}
This can be applied to the virtual constraint cost \eqref{eq:fbl_acc_cost}:
\begin{align*}
    \bigg|\bigg| J_y \ddot{q}^\perp_y + \dot{J}_y \dot{q} \bigg|\bigg|^2 
    = \bigg|\bigg| \underbrace{J_y(I - J_s^\dagger J_s)}_{J_y^\perp(q)} \ddot{q} 
    + \underbrace{( \dot{J}_y + J_y J_c^\dagger \dot{J}_s )}_{\dot{J}_y^\perp(q,\dot{q})} \dot{q} \bigg|\bigg|^2,
\end{align*}
and the CLF derivative in the cost for the \eqref{eq:id-clf-qp-plus}: 
\begin{align*}
    \mathcal{J}_{\dot{V}}^\perp(x,\mathcal{X}) := L_G V(x) J_y(q) \ddot{q}^\perp 
    =
    L_G V(x) J_y^\perp(q) \ddot{q}.
\end{align*}

\newsec{Final Controller.}
\begin{table}[b!]
\centering
    \caption{Weights used in the \eqref{eq:id-clf-qp-plus} controller on hardware.}
    \begin{tabular}{|l|l|l|}
        \hline
        Parameter & Value  \\ \hline
        $Q$        &   $[4600, 3640, 390, 4575, 8580, 4056, 1872, 520, \dots$  \\
                   & \hspace{10mm}$\dots 520, 16, 7.3, 1.6, 56, 115, 28.8, 18, 15, 12]$  \\ \hline
        $R$        &   $[0.8, 1, 1.4, 0.7, 0.8, 1, 1, 1.4, 1]$                      \\ \hline
        $w^{\mathrm{reg}}_{\ddot{q}_b}$                 &  $[0.01, 0.01, 0.01, 0.01, 0.01, 0.01]$             \\ \hline
        $w^{\mathrm{reg}}_{\ddot{q}_{l,\mathrm{st}}}$   &  $[0.01, 0.01, 0.01, 0.01, 0.6, 0.01, 0.6, 0.01]$   \\ \hline
        $w^{\mathrm{reg}}_{\ddot{q}_{l,\mathrm{sw}}}$   &  $[0.01, 0.01, 0.01, 0.01, 0.01, 0.01, 0.01, 0.01]$ \\ \hline
        $w^{\mathrm{reg}}_{u_{\mathrm{st}}}$            &  $[1, 0.9, 0.5, 0.1, 1]$                            \\ \hline
        $w^{\mathrm{reg}}_{u_{\mathrm{sw}}}$            &  $[1, 1, 0.9, 0.8, 1]$                              \\ \hline
        $w^{\mathrm{reg}}_{\mathrm{grf}}$               &  $[0.9, 0.1, 1.9, 1.3, 1.3]$                        \\ \hline
        $w^{\mathrm{reg}}_{\text{spring force}}$        &  $1.0$                                              \\ \hline
        $w_{\mathrm{grf}}$         &  $[1, 1, 1, 1.3, 1.3]$ \\ \hline
        $w_{\text{static spring}}$ &  $1.0$                 \\ \hline
        $w_y$                      &  $1.42$                \\ \hline
        $w_{\dot{V}}$              &  $1.40$                \\ \hline
    \end{tabular}
    \label{tab:clf_regularization_param}
\end{table}
The resulting controller that was implemented on hardware was posed in the form of \eqref{eq:id-clf-qp-plus}, with $\mathcal{X} = \left[\ddot{q}^T, u^T, \lambda_{s}^T \right]^T \in \mathbb{R}^{39}$:
\begin{align*}
    \mathcal{X}^{\ast} = \argmin_{\mathcal{X}\in \Xext} \quad |&|A(x) \mathcal{X} - b(x)||^2 + \dot{V}(q,\dot{q},\ddot{q}) \\
    \textrm{s.t.} \quad &  
    D_c(q) \ddot{q} + H_c(q,\dot{q}) \dot{q} = B_c(q) u + J_{c,s}^T(q) \lambda_{s}\\
    & \lambda_s \in \mathcal{AC}_{\mathrm{ss}}(\mathcal{X}) \\
    & u_\mathrm{lb} \leq u \leq u_\mathrm{ub} \\
    & u_{\text{s,ak}} = 0
\end{align*} 
where $u_{\text{s,ak}}$ is the stance ankle torque, which constrained to zero in order to ensure that the QP leaves it passive. 
The cost function is composed of the CLF derivative: 
\begin{align}
    \dot{V}(q,\mathcal{X}) := \begin{bmatrix} L_G V_\epsilon (x) J_y^\perp (q) & 0 & 0 \end{bmatrix} \mathcal{X},
\end{align}
and the least-squares cost terms for regularization and the virtual and soft constraints:
\begin{align*}
    A(x) = \begin{bmatrix}
        w_{\mathrm{reg}} I \\
        w_y J_y^\perp(q) \\
        w_{\lambda} J_s(q)
    \end{bmatrix}, \ b = \begin{bmatrix}
          w_{\mathrm{reg}} \mathcal{X}^*_\alpha(t,\bar{v}^a_{k-1}) \\
          w_y \left( \ddot{y}^d(t) - \dot{J}_y^\perp (q,\dot{q}) \dot{q} \right) \\
        - w_{\lambda} \dot{J}_s(q,\dot{q}) \dot{q}
    \end{bmatrix}. \label{eq:clf_walk_cost}
\end{align*}
The constraint feasibility associated with holonomic foot constraints in \secref{sec:robotmodel} are given as:
\begin{align*}
    \mathcal{AC}_{\mathrm{ss}}(\mathcal{X}) = \begin{bmatrix} 
        \{ \lambda_{\mathrm{sf}}^z, \lambda_{\mathrm{sf}}^z \} \\
        \frac{\mu}{\sqrt{2}} \lambda_{\mathrm{sf}}^z - \{ |\lambda_{\mathrm{sf}}^x|, |\lambda_{\mathrm{sf}}^y| \} \\
        \frac{l_f}{2} \lambda_{\mathrm{sf}}^z - \{ |\lambda_{\mathrm{sf}}^{my}|, |\lambda_{\mathrm{sf}}^{mz}| \} 
    \end{bmatrix} \geq 0, 
\end{align*}
corresponding to the friction pyramid \eqref{eq:pyramid_friction} and foot rollover \eqref{eq:footroll}. 


%
\begin{figure}[t!]
	\centering
	\includegraphics[width=1\columnwidth]{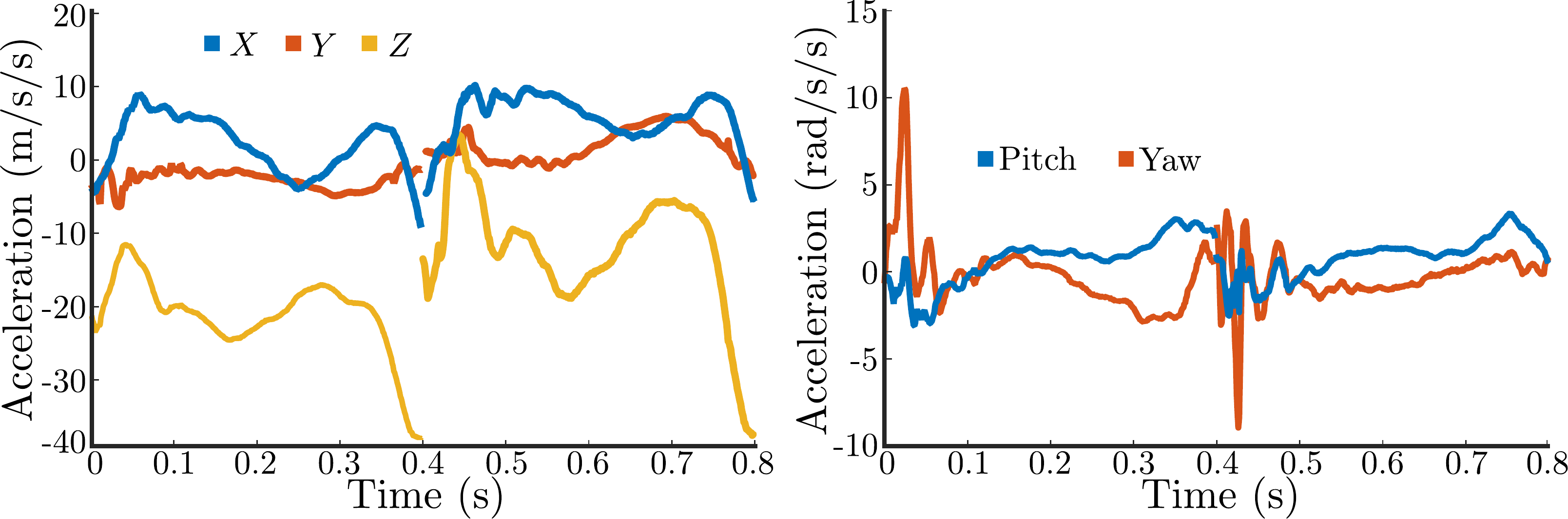}
	\caption{Shown are the acceleration violations of the soft constraint on the stance foot contact over two steps. The largest error is in the vertical direction, which corresponds to our observation that the robot mass has some inaccuracy.}
	\label{fig:cassie_clf_foot_accel_violation}
\end{figure}
%

\subsection{Regulation}
Directly implementing the motion library obtained from \eqref{eq:opteqs} can at best result in a marginally stable locomotion, as it is always operating on the orbit of the current walking speed, and has no ability to track a target walking speed. 
Motivated by this, a regulator is applied to modify the nominal accelerations of the floating base accelerations:
\begin{align}
    \ddot{q}^d_{x,y} = \ddot{q}^*_{x,y} + {k}_{\ddot{q},p} \left(p^a_{x,y}(q) - p^*_{x,y}\right) + 
        k_v (\tilde{v}_{x,y}^{a} - v^d_{x,y})
\end{align}
where ${k}_v^{x,y}=[3, 3]$ and ${k}_{\ddot{q},p}^{x,y}=[2, 2]$.
An additional regulator is used to find an offset to the footstrike location \cite{raibert1984experiments}:
\begin{align*}
    \Delta_{(x,y)} &= \tilde{K}_{p}^{x,y} (\tilde{v}^{a} - v^d) + \tilde{K}_{d}^{x,y} (\tilde{v}^{a} - \bar{v}_{k-1}^a)  \\
    &\hspace{20mm} + k_{i} \int_0^t \gamma (\tilde{v}^{a}(t') - v^d(t')) dt',
\end{align*}
where $\tilde{K}_p^{x,y} = [0.045, 0.0375]$, $\tilde{K}_d^{x,y}=[0.18, 0.21]$, $\tilde{K}_i^{x,y}=[0.06, 0]$, and $\gamma = 0.9995$. We can find the current step velocity as $\tilde{v}_{x,y}^a = \bar{v}_{k-1}^a + \left( v^a_{x,y} - v^*_{x,y} \right)$, where $v^d_{x,y}$ is the target step velocity from the user joystick and $v^a_{x,y}$ is the instantaneous velocity of the robot. 
We then define $\Delta := \left[ \Delta_x, \Delta_y, 0 \right]^T$ and augment the nominal value of the desired outputs as:
\begin{align*}
        y_{sw,ll}^d &= || p_{nsf}^*(y^d) + \Delta ||_2, \\
    y_{lp}^d &= \sin^{-1} \left( \frac{p_x^*(y^d) + \Delta_x}{y_{sw,ll}^d} \right) - y_{b,x}^d(t,\alpha), \\
    y_{lr}^d &= \sin^{-1} \left( \frac{p_y^*(y^d) + \Delta_y}{y_{sw,ll}^d} \right) - y_{b,y}^d(t,\alpha),
\end{align*}
where $p_{nsf}^*(y^d) = ( p_{nsf,x}^*, p_{nsf,y}^*, p_{nsf,z}^* )$ are the nominal Cartesian swing foot positions.
Finally, to help overcome energy loss from disturbances, an additional regulator (adapted from \cite{rezazadeh2015spring}) applies a corrective force to the leg length:
\begin{align*}
    \Delta F_l &= - k_{p,F} y_{2,\mathrm{sll}}(t,q) - k_{d,F} \dot{y}_{2,\mathrm{sll}}(t,q,\dot{q}) \\
    &\hspace{10mm} + 
    \begin{cases}
        \left[ k_{\lambda,1} v^d_x + k_{\lambda,2} (v^d_x - \bar{v}_{x,k-1}^a) \right] \frac{x}{x_0} & \text{if} \ x\leq 0 \\
        \left[ k_{\lambda,3} v^d_x + k_{\lambda,4} (v^d_x - \bar{v}_{x,k-1}^a) \right] \frac{x}{x_0} & \text{if} \ x> 0
    \end{cases}
\end{align*}
where $x$ is the forward position within the stride, ${k}_{D,L}, {k}_{D,L}=[1800, 22]$, and $[k_{\lambda,1}, k_{\lambda,2}, k_{\lambda,3}, k_{\lambda,4}]=[0.025, 0.04, 0.0075, 0]$. The QP is then regularized with the $\lambda^* = \lambda_\alpha^*(t) + R_l(q) \Delta F_l$ where $R_l(q)$ is a rotation matrix which maps the axial force into an Cartesian force at the foot.

\begin{figure}[t!]
	\centering
	\includegraphics[width=1\columnwidth]{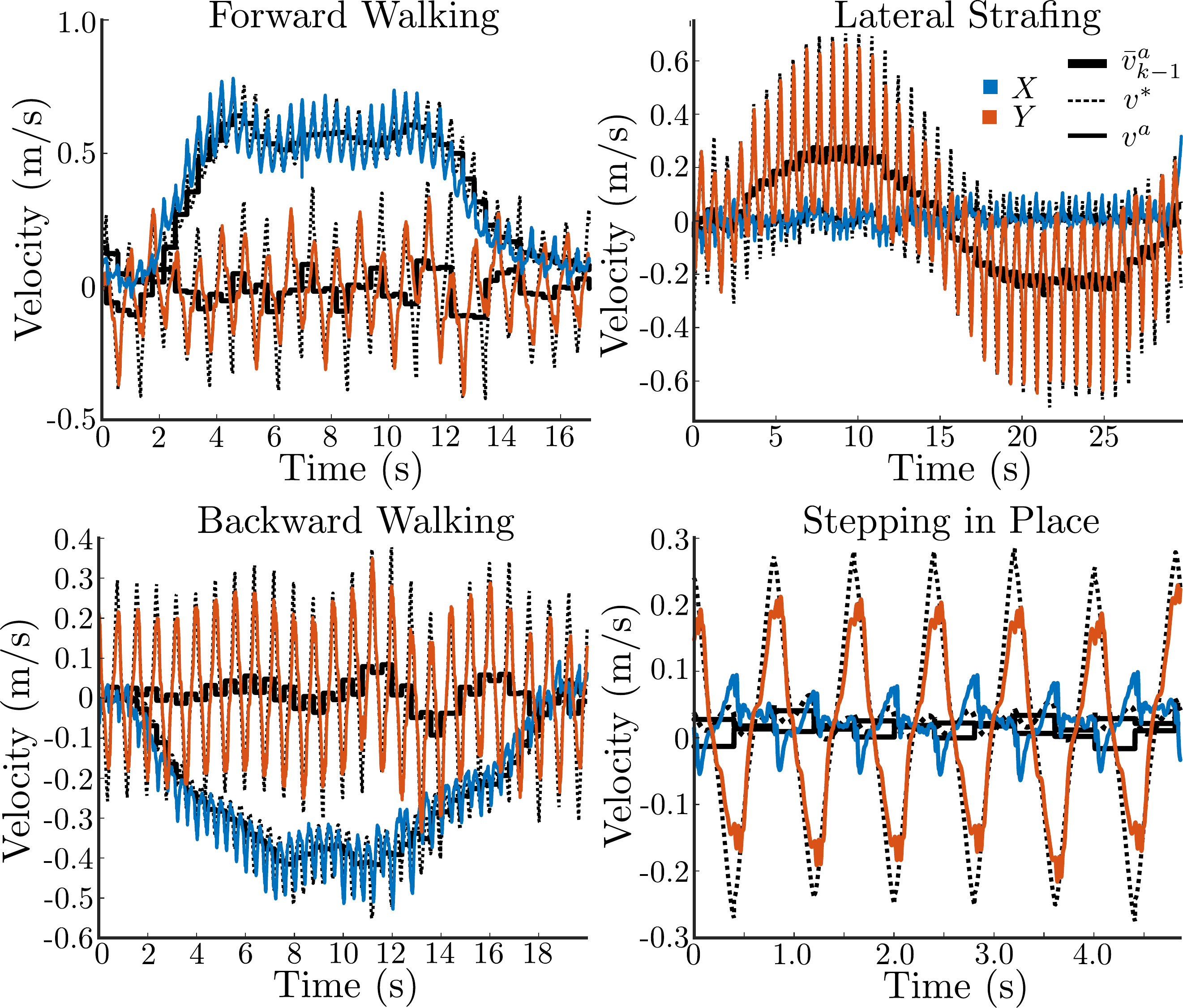}
	\caption{Shown is a comparison of the desired velocities from the current gait for the motion library, $\bar{v}_{k-1}^a$, with the actual velocity for different behaviors.}
	\label{fig:cassie_clf_velocity_tracking}
\end{figure}
\begin{figure}[t!]
	\centering
	\includegraphics[width=1\columnwidth]{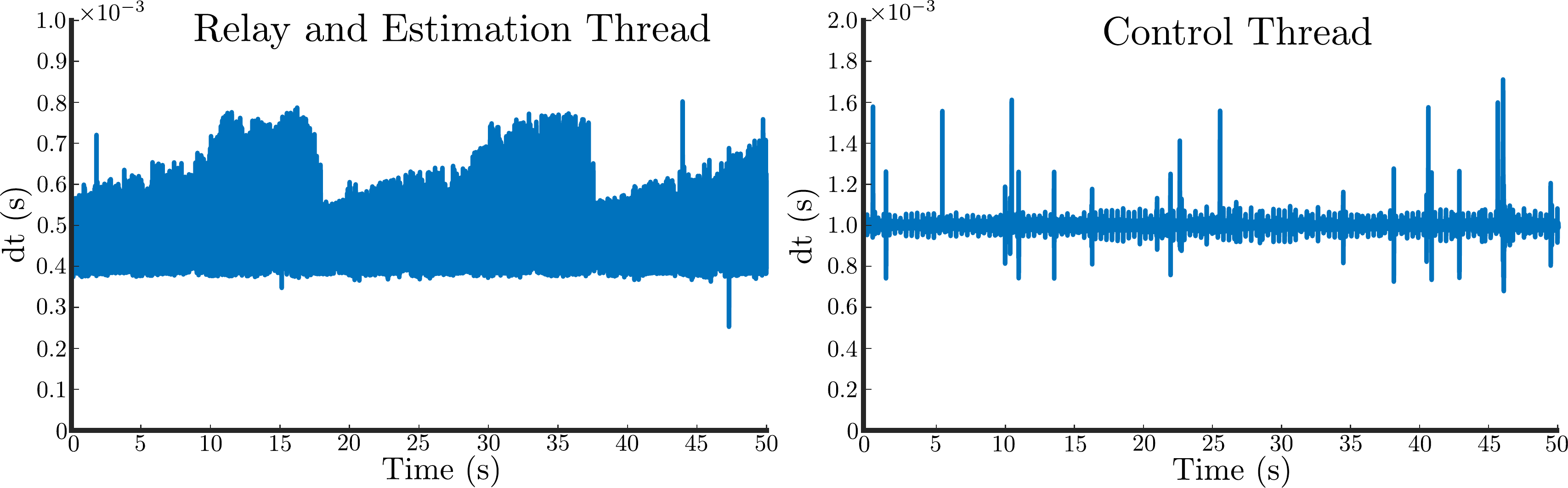}
	\caption{Thread timing for the estimation and communication relay node ($2$ kHz) and feedback control node ($1$ kHz) for $50$s of walking on hardware.}
	\label{fig:cassie_clf_dt}
\end{figure}
%


\begin{figure*}[t!]
	\centering
	\includegraphics[width=1\textwidth]{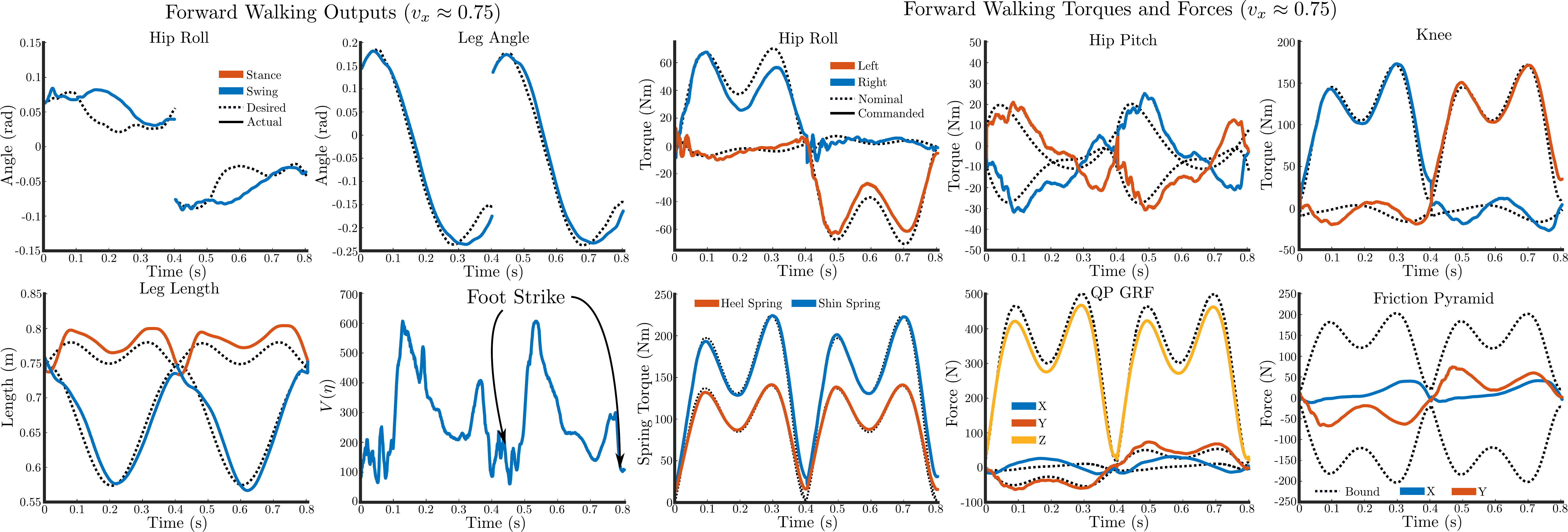}
	\caption{(Left) Output tracking for forward walking on Cassie, with the CLF evolution over two steps also shown. (Right) Inputs selected by the ID-CLF-QP which are applied to Cassie for forward walking. Pictured are torques and contact forces, along with the stance spring forces and friction pyramid constraint.}
	\label{fig:cassie_clf_forwardwalk_outputs}
	\vspace{-3mm}
\end{figure*}
\begin{figure}[t!]
	\centering
	\includegraphics[width=0.95\columnwidth]{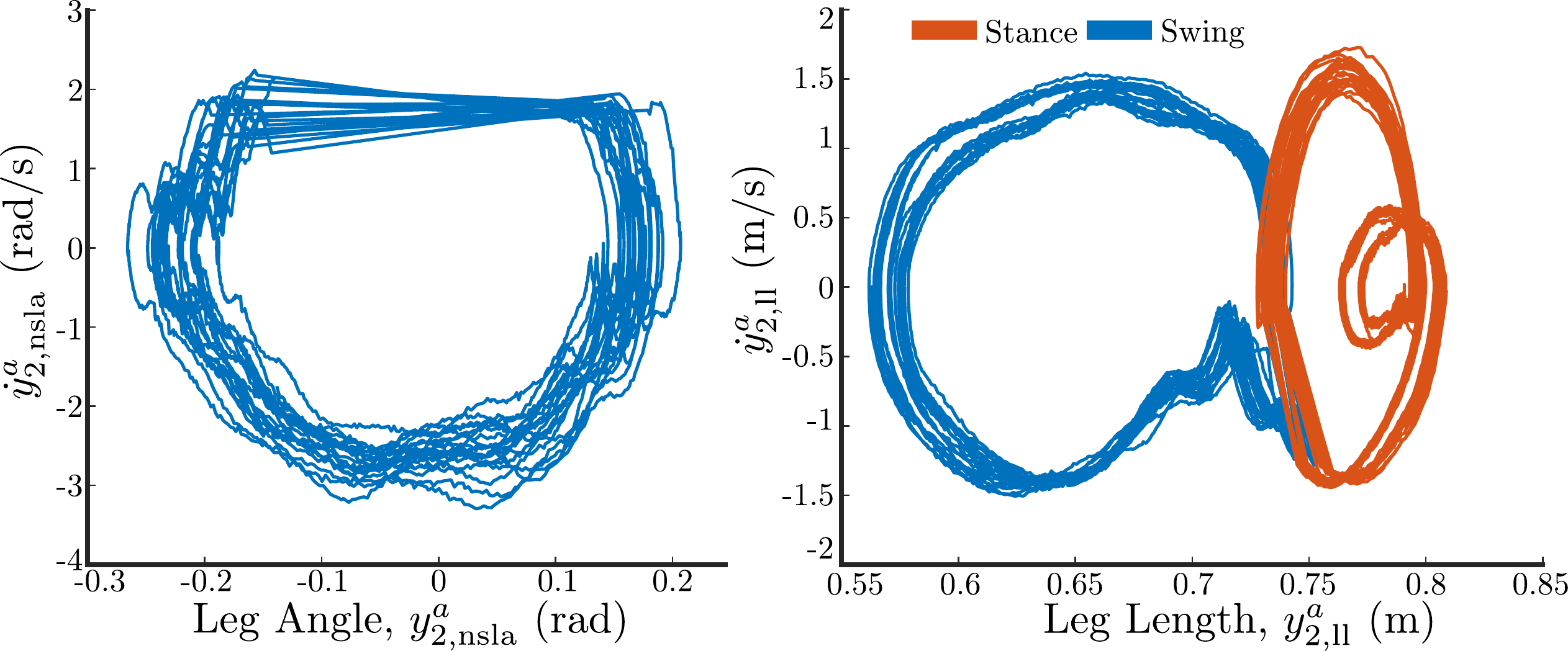}
	\caption{Phase portrait of the leg angle and leg length outputs while walking.}
	\label{fig:cassie_limit}
	\vspace{-3mm}
\end{figure}

\subsection{Software Architecture}
The software for feedback control and was implemented in C++ on the Intel NUC computer in the Cassie torso, and runs on two ROS nodes: one which communicates to the Simulink xPC over UDP to relay torques and sensor data and to perform estimation, and a second which runs the controllers. 
The first node runs at $2$ kHz and executes contact classification, inverse kinematics to obtain the heel spring deflection, and an EKF for velocity estimation \cite{bloesch2013state,reher2019dynamic}.
The second node runs at $1$ kHz and executes the \eqref{eq:id-clf-qp-plus} controller using qpOASES. 
The evaluation time during walking is shown in \figref{fig:cassie_clf_dt}, and the general control concept is illustrated in \figref{fig:cassie_clf_forward_tiles}.

\section{Results} \label{sec:results}
\begin{figure}[t!]
	\centering
	\includegraphics[width=1\columnwidth]{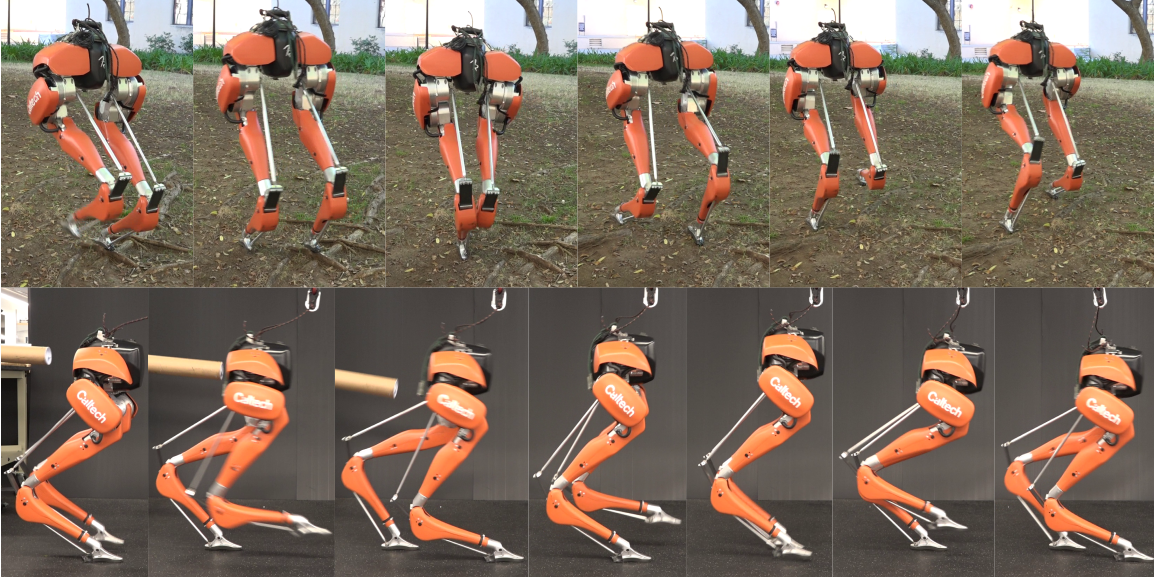}\\
	\vspace{1mm}
	\includegraphics[width=1\columnwidth]{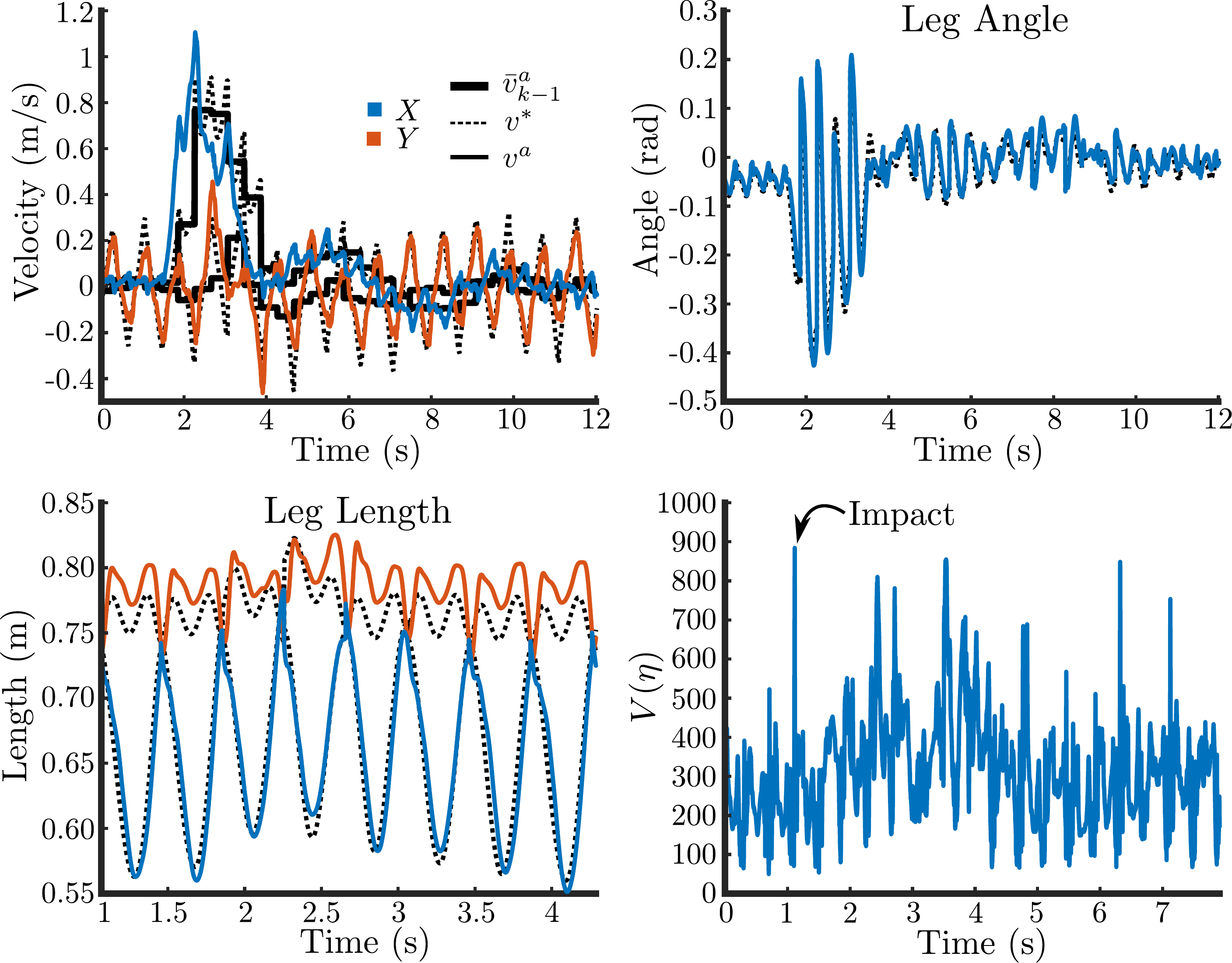}
	\caption{Perturbed walking on Cassie, with the robot walking over a series of roots outdoors and being subjected to a push from behind. Plotted on the bottom is the velocity, outputs, and CLF values during the push.}
	\label{fig:cassie_clf_disturbance_tiles}
\end{figure}
\begin{figure*}[t!]
	\centering
	\includegraphics[width=0.94\textwidth]{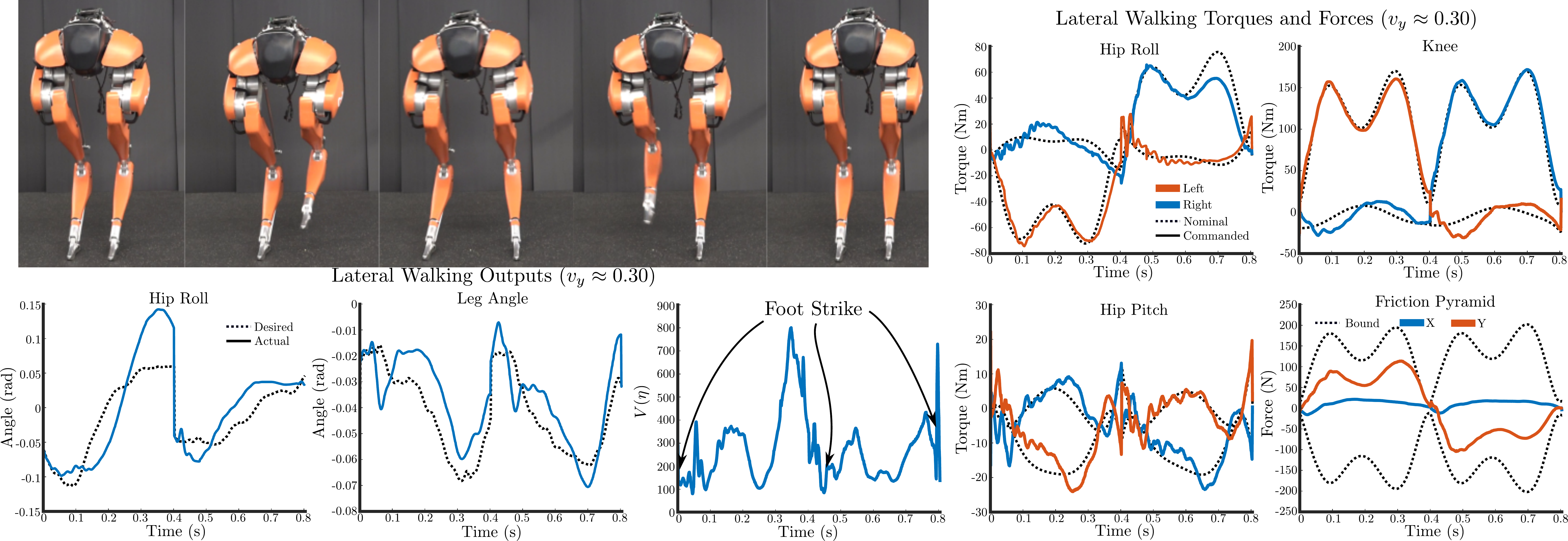}
	\caption{
	(Left) Output tracking for lateral walking on Cassie, with gait tiles and the CLF evolution over two steps also shown. The period-$2$ nature of the walking is apparent in the reference polynomials, while the Lyapunov function is shown to have a significant relaxation at the end of the first step.  (Right) Inputs selected by the controller which are applied to Cassie. Pictured are torques and contact forces, along with the stance spring forces and friction pyramid. }
	\label{fig:cassie_clf_sidewalk}
	\vspace{-3mm}
\end{figure*}
The results presented in this section demonstrate stable walking on Cassie over a variety of terrains, for walking in all directions, and when subjected to a push disturbance while stepping in place. 
Overall, these results illustrate the \textit{ability of the \eqref{eq:id-clf-qp} control framework to stabilize underactuated HZD locomotion while accounting for passive compliance}. 
Walking was first performed indoors for each of the behaviors to isolate segments of walking data to plot. The robot was then taken outdoors where it walked on sidewalks and over rough terrain. A video of these experiments is provided \cite{cassieclf}, with gait tiles shown in \figref{fig:cassie_clf_forward_tiles}. 
The velocity tracking for each of the motions is shown in \figref{fig:cassie_clf_velocity_tracking}, where the dashed line is the nominal value from the optimization and the solid line is the estimated value obtained from an EKF running onboard. 

The virtual constraint tracking for forward walking is depicted on the left in \figref{fig:cassie_clf_forwardwalk_outputs}, where reasonable tracking performance is shown. The CLF value is also provided, which approaches it's largest value near mid-step but converges back to a low value near the end of the step and prior to impact. In addition, the commanded torques and constraint forces for the same steps during a forward walking motion are given on the right in \figref{fig:cassie_clf_forwardwalk_outputs}. In each of the torque and force plots, we can see that the nominal value from optimization closely matches. The torques have some minor oscillations near impact when the velocities of the system jump, but are otherwise smooth. Finally, the friction pyramid is shown with the bound becoming very small near impact due to the low vertical contact force. 

The same output, CLF, force, and torque plots for lateral walking are provided in \figref{fig:cassie_clf_sidewalk} along with gait tiles of the corresponding motion. These plots closely resemble the results that were discussed for the forward walking case, with the primary difference being the period-$2$ orbital nature of the lateral walking. However, the lateral walking plots show that the friction pyramid is triggered for a portion of the first step in the data that was plotted. This is particularly interesting because we can see that when the friction pyramid constraint is triggered, the CLF value rises as it under-performs with tracking. 

The robot also performed well under disturbances, both in the form of terrain and pushes. Motion tiles for Cassie walking outdoors over a system of roots and while being pushed aggressively are shown in \figref{fig:cassie_clf_disturbance_tiles} along with the corresponding velocity, outputs, and CLF values. The push drives the robot forward to approximately $1$ m/s, after which the robot recovers to a near-zero speed in roughly $5$ steps. 

These results demonstrate the first successful experimental walking with CLFs on a 3D biped that the authors are aware of to date. The accuracy of the robot model and synthesized motion library allowed the control implementation to use a regularization term as described in \propref{prop:regularization_id-clf}. The use of this regularization, combined with the relaxed form of the \eqref{eq:id-clf-qp-plus} allowed for smooth torques and force references which would have been significantly more complex to implement on a \eqref{eq:clf-qp-delta} controller.

\section{Conclusion}
In conclusion, this paper has presented an approach to modeling, trajectory planning and parameterization, and real-time control for compliant walking on a Cassie biped at a variety of walking speeds. 
An HZD based controller was developed in \secref{sec:trajectory_hzd} for Cassie to synthesize closed-loop locomotion plans while leveraging the passive compliance in the robot model to generate trajectories that accurately reflect how the true compliant robotic system would evolve on hardware.  
The resulting optimization results form a compliant HZD motion library that leverages its full-body dynamics, including passive compliance, and also encodes nominal information with respect to accelerations and contact forces. 
A feedback control approach which couples convergence constraints from control Lyapunov functions with desirable formulations from task-based inverse dynamics control and quadratic programming approaches is then shown, along with a theoretical analysis demonstrating several useful properties of the approach for tuning and implementation. Further, the stability of the controller for HZD locomotion is proven.
This was extended to a relaxed version of the CLF controller, which removes a convergence inequality constraint in lieu of a conservative CLF cost within a quadratic program to achieve tracking. 
Finally, this end-to-end approach to locomotion was implemented on hardware. 
The results indicate that the control method, when combined with the HZD motion library and parameterized optimization results, led to smooth input torques and feasible 
GRF on hardware with reliable output tracking. 


%
%

\section*{Acknowledgment}
This research was supported under NSF Grant Numbers 1544332, 1724457, 1724464 and Disney Research LA. The authors would like to thank Claudia Kann for her help in the early discussions of the ID-CLF-QP approach, and Wen-Loong Ma and Noel Csomay-Shanklin for experimental assistance.

\ifCLASSOPTIONcaptionsoff
  \newpage
\fi


\bibliographystyle{plain}
\bibliography{IEEEabrv, references.bib}

%





\end{document}